\documentclass[12pt,reqno,oneside]{amsart}
\usepackage[latin1]{inputenc}
\usepackage[T1]{fontenc}
\usepackage{amsmath,amsfonts,amscd,amssymb,amsthm}
\usepackage{mathtools}
\usepackage{dsfont}
\usepackage{color}
\usepackage[mathscr]{euscript}
\usepackage{comment}
\usepackage{graphicx}
\usepackage{placeins}
\usepackage{extsizes}
\usepackage{float}
\usepackage{multirow}
\usepackage{pdflscape}
\usepackage{caption}
\usepackage{pgfplots}
\usepackage{subcaption}
\usepackage{mathabx}
\usepackage{tikz}
\usepackage{hyperref}
\usetikzlibrary{decorations.pathreplacing}
\usetikzlibrary{backgrounds} 
\usetikzlibrary{calc}
\pgfdeclarelayer{bg}    
\pgfsetlayers{bg,main}
\usepackage{rotate}
\usepackage[noend]{algorithmic}
\usepackage{algorithm}
\usepackage{setspace}
\definecolor{color1bg}{HTML}{f73d28}
\definecolor{color2bg}{HTML}{FA8072}
\definecolor{bblue}{HTML}{00BFFF}
\definecolor{bblue2}{HTML}{00ffff}

\usetikzlibrary{arrows,calc}
\tikzset{
	>=stealth',
	help lines/.style={dashed, thick},
	axis/.style={<->},
	important line/.style={thick},
	connection/.style={thick, dotted},
}

\usetikzlibrary{shadows}
\usetikzlibrary{backgrounds}
\tikzset{
	diagonal fill/.style 2 args={fill=#2, path picture={
			\fill[#1, sharp corners] (path picture bounding box.south west) -|
			(path picture bounding box.north east) -- cycle;}},
	reversed diagonal fill/.style 2 args={fill=#2, path picture={
			\fill[#1, sharp corners] (path picture bounding box.north west) |- 
			(path picture bounding box.south east) -- cycle;}}
}
\usetikzlibrary{arrows.meta}

\makeatletter
\@addtoreset{equation}{section}
\makeatother
\newcounter{as}[section]

\DeclareMathOperator*{\argminA}{arg\,min}

\title[Error Rates for Physics-informed Statistical Learning]{Complexity Dependent Error Rates for Physics-informed Statistical Learning via the Small-ball Method}

\author{Diego Marcondes\\ $\;$ \\\today}
\address{Mathematical Sciences Institute and France-Australia Mathematical Sciences and Interactions ANU-CNRS International Research Lab, The Australian National University  \\
	e-mail: \texttt{diego.marcondes@anu.edu.au}}

\allowdisplaybreaks
\newtheorem{theorem}{Theorem}[section]
\newtheorem{remark}[theorem]{Remark}

\newtheorem{definition}[theorem]{Definition}

\newtheorem{lemma}[theorem]{Lemma}
\newtheorem{proposition}[theorem]{Proposition}

\newtheorem{assumption}[theorem]{Assumption}

\newcommand{\mc}[1]{{\mathcal #1}}

\newcommand{\ms}[1]{{\mathscr #1}}

\newcommand{\ustar}{u^{\star}}

\newcommand{\hhat}{\hat{h}_{n,\Psi}}
\newcommand{\fstar}{f^{\star}}
\newcommand{\hstar}{h^{\star}}
\newcommand{\cH}{\text{conv } \mathscr{H}}
\newcommand{\Hstar}{\mathscr{H}^{\star}}
\newcommand{\tPsi}{\widetilde{\Psi}}

\usepackage{lineno}

\definecolor{bblue}{rgb}{.2,0.2,.8}

\begin{document}
	\maketitle
	
	\begin{abstract}
		Physics-informed statistical learning (PISL) integrates empirical data with physical knowledge to enhance the statistical performance of estimators. While PISL methods are widely used in practice, a comprehensive theoretical understanding of how informed regularization affects statistical properties is still missing. Specifically, two fundamental questions have yet to be fully addressed: (A) what is the trade-off between soft penalties and hard constraints, and (B) what is the statistical gain of incorporating physical knowledge compared to purely data-driven empirical error minimisation. In this paper, we address these questions for PISL in classes of functions whose convex hull satisfies a small-ball condition, in particular convex classes, under physical knowledge expressed as linear equations by developing appropriate complexity dependent error rates based on the small-ball method. We specialise them to linear classes of finite dimension $q$ in which the space of solutions of the equation has dimension $k < q$, and show that (A) the physics-informed estimator attains the error bound of the hard-constrained empirical error minimiser whenever the penalty parameter is above an explicit threshold, and that (B) the penalty improves the error rate from $\sqrt{q/n}$ to $\sqrt{k/n}$ under a second moment condition on the noise only, having the effect of a dimensionality reduction. We also quantify how these conclusions degrade when the equation is misspecified, in terms of the excess residual of the target function. This work establishes a theoretical framework for evaluating the statistical properties of physics-informed estimators in such classes, contributing to closing the gap between statistical theory and practical PISL.
		
		\vspace{0.25cm}
	\end{abstract}

\section{Introduction}

Let $X$ be a random variable taking values in a set $\Omega \subset \mathbb{R}^{d}$ with distribution $\mu$ and $Y = \ustar(X) + e$ for $\ustar$ in the Sobolev space $H^{s}(\mu)$ and a mean zero random variable $e$ independent of $X$. Although $\ustar$ is unknown, we assume there is prior knowledge that it satisfies a differential equation
\begin{linenomath}
	\begin{align}
		\label{dif_eq}
		\mathscr{D}\ustar = g 
	\end{align}
\end{linenomath}
for a linear differential operator $\mathscr{D}: H^{s}(\mu) \mapsto L_{2}(\mu)$ and $g \in L_{2}(\mu)$. We are interested in the problem of estimating $\ustar$ based on an independent sample $S_{n} = \{(X_{1},Y_{1}),\dots,(X_{n},Y_{n})\}$ of $(X,Y)$ taking into account that $\ustar$ satisfies \eqref{dif_eq}.

For this, we first fix a class $\ms{H} \subset H^{s}(\mu)$ that we believe to have functions that well approximate $\ustar$. Formally, defining $\hstar \coloneqq \argminA_{h \in \ms{H}} \mathbb{E}(h(X) - Y)^{2} =  \argminA_{h \in \ms{H}} \lVert h - \ustar \rVert_{L_{2}(\mu)}$, which for now we assume is unique, it should hold $\lVert \hstar - \ustar \rVert_{L_{2}(\mu)} \approx 0$. Then, we consider estimators that minimise the \textit{physics-informed} mean squared error
\begin{linenomath}
	\begin{equation}
		\label{PINN_problem}
		L_{n,\Psi}(h) = L_{n}(h) + \lambda \ \Psi(h) \coloneqq \frac{1}{n} \sum_{i=1}^{n} (h(X_{i}) - Y_{i})^{2} + \lambda \ \lVert \mathscr{D}h - g \rVert_{L_{2}(\mu)}^{2}
	\end{equation}
\end{linenomath}
for $\lambda > 0$ fixed and $h \in \ms{H}$, which is obtained by adding a \textit{soft penalty} $\Psi(h)$ to the empirical error $L_{n}(h)$, that penalizes functions which deviate from \eqref{dif_eq}. These are estimators $\hhat$ satisfying
\begin{linenomath}
	\begin{equation}
		\label{informed_est}
		\hhat \in \argminA_{h \in \ms{H}} L_{n,\Psi}(h).
	\end{equation}
\end{linenomath}

From a theoretical perspective, physics-informed penalization is not the natural way of taking into account that $\ustar$ satisfies \eqref{dif_eq} since, in principle, one could obtain an estimator by minimising the empirical error in the set $\mathscr{H}_{\Psi,\epsilon} \coloneqq \left\{h \in \mathscr{H}: \Psi(h) \leq \epsilon\right\} \neq \emptyset,  \epsilon \geq 0$, of the functions in $\mathscr{H}$ that approximately satisfy \eqref{dif_eq}, which should be a statistically better estimator than that obtained by minimising the plain empirical error $L_{n}(h)$ over $\mathscr{H}$. In this case, prior information \eqref{dif_eq} yields a hard constraint on the set $\ms{H}$. However, in practice, this may not be computationally viable and informed regularization could be an alternative so the prior information $\Psi(\ustar) = 0$ is not disregarded. 

In this context, there are two important questions about the statistical properties of the error $\lVert \hhat - \hstar \rVert_{L_{2}(\mu)}$ of physics-informed estimators:
\begin{itemize}
	\item[] \textbf{(A):} \textit{What is the price we pay by considering a soft penalty instead of a hard constraint on $\mathscr{H}$, that is, how the error rate of $\hhat$ compares to that of empirical error minimisers of $\mathscr{H}_{\Psi,\epsilon}$?}
	\item[] \textbf{(B):} \textit{What is the statistical effect of physics-informed penalization, that is, how the error rate of $\hhat$ compares to that of plain empirical error minimisers of $\mathscr{H}$?}
\end{itemize}
These are important questions in physics-informed statistical learning that are open problems for general classes $\mathscr{H}$. In this paper, we deduce high probability bounds for $\lVert \hhat - \hstar \rVert_{L_{2}(\mu)}$ based on the small-ball method of \cite{mendelson2015learning}, taking a step towards answering these questions in the case of linear operators $\ms{D}$ and classes of functions $\ms{H}$ whose convex hull satisfies a small-ball assumption (cf. Assumption \ref{small_ball}).

\subsection{Physics-informed statistical learning} 
\label{sec_PISL}

The problem of minimising \eqref{PINN_problem} is a special case of physics-informed statistical learning (PISL) which combines empirical data with physics knowledge to improve the performance of estimators. The foremost PISL method are the physics-informed neural networks (PINNs) when $\mathscr{H}$ is generated by a neural network architecture. PINNs were proposed by \cite{raissi2019physics} for solving forward and inverse problems related to PDEs and are usually applied as meshless numerical solvers.

As an example of PISL, consider the forward problem
\begin{linenomath}
	\begin{align}
		\label{forward}
		\begin{cases}
			\mathscr{D}u = g, & \text{ in } \Omega_{x} \times (0,T) \\
			\mathscr{B}u = b, & \text{ in } \partial\Omega_{x} \times (0,T]\\
			u = u_{0}, &  \text{ in } \bar{\Omega}_{x} \times \{0\}
		\end{cases}
	\end{align}
\end{linenomath}
in which $\Omega = \Omega_{x} \times (0,T)$, $\mathscr{B}: H^{s}(\mu_{\partial}) \mapsto L_{2}(\partial \Omega_{x} \times (0,T],\mu_{\partial})$ represents boundary conditions, $b \in L_{2}(\partial \Omega_{x} \times (0,T],\mu_{\partial})$, and $u_{0} \in L_{2}(\bar{\Omega}_{x},\mu_{x})$ is the initial condition, assuming that $\Omega_{x} \subset \mathbb{R}^{d}$ is an open set, $\bar{\Omega}_{x}$ is compact, $T  > 0$ is fixed, and $\mu_{\partial}$ and $\mu_{x}$ are suitable measures, e.g., Lebesgue measure.

If the solution of \eqref{forward} is unique, then it is the unique solution of
\begin{linenomath}
	\begin{equation}
		\label{complete_pen}
		\lambda \ \lVert \mathscr{D}u - g \rVert_{L_{2}(\mu)}^{2} + \lambda^{b} \ \lVert \mathscr{B}u - b \rVert_{L_{2}(\mu_{\partial})}^{2} + \lambda^{0} \ \lVert u - u_{0} \rVert_{L_{2}(\mu_{x})}^{2} = 0
	\end{equation}
\end{linenomath}
for any $\lambda,\lambda^{b},\lambda^{0} > 0$ fixed. Therefore, if the boundary and initial conditions are known, the solution of \eqref{forward} might be approximated by minimising the left-hand side of \eqref{complete_pen} for $u$ in a rich class of functions $\ms{H}$, e.g., generated by a neural network architecture. We refer to the reviews \cite{cuomo2022scientific,karniadakis2021physics,meng2025physics,von2021informed} for more details about PINNs.

However, in inverse problems where boundary or initial conditions are unknown, the numerical approximation problem \eqref{complete_pen} is ill-posed and data on the solution is necessary. In this instance, approximately solving \eqref{forward} becomes a statistical problem that might be solved by a physics-informed estimator that minimises \eqref{PINN_problem} by adding the respective penalty to the regularization term $\Psi(h)$ if either boundary or initial conditions are known. Unlike forward problems, in this case there is incomplete information about $\ustar$ that must be combined with data, hopefully allowing to solve \eqref{forward} approximately.

This class of inverse problems is important, since in practice initial and boundary conditions are hardly known, but data on the solution may be measured over time by sensors placed in $\Omega_{x}$, which return noisy data of the form $\ustar(X,t) + e$. Moreover, in this case, physics-informed estimators are usually among the few practical learning methods\footnote{We note that there are many numerical analysis methods for recovering $\ustar$ in inverse problems \cite{hasanouglu2021introduction}, but here we are concerned with learning methods only.} to approximate $\ustar$ taking into consideration that it satisfies \eqref{dif_eq}, apart from probabilistic numerical methods such as physics-informed Gaussian processes (see for example \cite[Chapter~12]{braga2024fundamentals}). For instance, computational constraints may prevent efficiently minimising the empirical error $L_{n}$ in the set
\begin{linenomath}
	\begin{equation}
		\label{H_psi}
		\ms{H}_{\Psi} \coloneqq \left\{h \in \ms{H}: \Psi(h) \leq \Psi(h^{\star})\right\},
	\end{equation}
\end{linenomath}
even if $\Psi(h^{\star})$ is known, e.g., when $\hstar$ satisfies \eqref{dif_eq} and $\Psi(\hstar) = 0$, so adding a soft penalty to the empirical error is a more efficient procedure than hard constraints to take into account that the data generating function $\ustar$ satisfies \eqref{dif_eq}.

\subsection{The statistics of physics-informed regularization}

Even though PINNs are increasingly prevalent in the literature to solve inverse problems in science and engineering (see for example \cite{chen2020physics,jagtap2022physics}), the statistical properties of physics-informed regularization are not completely understood, even in simpler models. Actually, since the proposal of PINNs, their approximation error as numerical solvers of forward problems has been investigated for many classes of PDEs and neural network architectures (see \cite{de2024error2,de2022error,hu2022xpinns,mishra2023estimates,qian2023physics,shin2020convergence,shin2023error,wu2022convergence} and the references therein), but the statistical effect of penalizing the empirical error on sensor data by the known conditions of the PDE, an approach often called \textit{hybrid modelling} \cite{rai2020driven,von2023does}, has not been thoroughly investigated. 

The statistical properties of physics-informed estimators have been investigated in particular cases by \cite{arnone2022some,doumeche2023convergence,ferraccioli2022some}. The consistency of the estimator in spatial regression with general linear second-order PDE regularization has been established in \cite{arnone2022some} and inference tools for them were studied in \cite{ferraccioli2022some}. A review of this type of regularized spatial regression can be found in \cite{sangalli2021spatial}. In \cite{doumeche2023convergence} it is highlighted that solving \eqref{PINN_problem} in some cases may lead to overfitting and showed that this can be avoided by ridge regularization. The authors showed, in specific cases, the error-consistency under ridge regularization and analysed the strong convergence of PINNs.

To the best of our knowledge, there are only three results in the literature so far implying that physics-informed regularization improves the rate of convergence of empirical error minimisation estimators. The first is \cite{doumeche2024physics} which treats physics-informed kernel regression and linear PDEs. They showed that the solution of \eqref{PINN_problem}, with a Sobolev penalization, for $h$ in the respective Sobolev space, is in a reproducing kernel Hilbert space (RKHS), with kernel dependent on the PDE, and hence the problem reduces to kernel regression. Error rates for $\lVert \hhat - \hstar \rVert_{L_{2}(\mu)}$ then follow from results for regularized least squares such as \cite{caponnetto2007optimal} by making a sub-Gamma assumption about the noise $e$. The application of the result of \cite{doumeche2024physics} requires the deduction of a kernel associated with the linear operator $\mathscr{D}$ that is specific to learning in the respective RKHS. Based on the theoretical results of \cite{doumeche2024physics}, practical methods for physics-informed kernel learning were proposed in \cite{doumeche2024physics2}.

The second is \cite{koshizuka2025understanding} which combines a soft penalty with variational methods to prove that the error rates in physics-informed regression are dependent on the dimension of an affine variety associated with the operator $\ms{D}$. Finally, the third is the recent work \cite{scampicchio2025physics}, which applies the complexity dependent bounds developed in \cite{lecue2017regularization} based on the small-ball method to establish error rates for PISL with dependent data for learning in Sobolev spaces when the operator $\ms{D}$ is elliptic and the forcing term in \eqref{dif_eq} is $g = 0$.

This paper differs from \mbox{\cite{doumeche2024physics,koshizuka2025understanding,scampicchio2025physics}} in scope. While they focus on specific classes of functions, such as linear regression or specific subsets of Sobolev spaces, we present a general theory for PISL that applies, for instance, to closed and convex classes and general linear operators $\ms{D}$ even when $g \neq 0$ under a small-ball assumption (cf. Assumption \ref{small_ball}). We apply this theory to linear classes of finite dimension, which cover many methods used in practice, for which we obtain explicit error rates that address \textbf{(A)} and \textbf{(B)}, also when the physics is misspecified. 

\subsection{Regularization in statistical learning}

Informed regularization, characterised by prior information that $\Psi(\ustar) = 0$ for the data generating function $\ustar$, differs in form and purpose from the usual regularization methods in statistical learning. Classical regularization methods, such as Ridge or Tikhonov regularization \cite{hoerl1970ridge,tikhonov1943stability} and LASSO \cite{tibshirani1996regression}, are concerned with increasing the stability and efficiency of estimators in instances where unregularized estimators tend to overfit the data and not generalise well. These regularization methods generate \textit{smoother} estimators with better statistical properties (see \cite{hastie2009elements} for more details).

From an ``informed'' perspective, this type of regularization, in which $\Psi$ is usually a norm, carries the implicit assumption that $\Psi(\ustar)$ is \textit{small} and the regularized estimator is obtained by minimising the empirical error $L_{n}(h)$ penalized by $\lambda\Psi(h)$ to \textit{avoid} functions with high values of $\Psi$ that are not counterbalanced by a low empirical error. However, since $\Psi$ is a norm, even if $\Psi(\ustar)$ is not \textit{small}, $\Psi(h)$ could be penalizing the lack of \textit{sparsity}, in some sense, of $h$, so if it is known or believed that $\ustar$ is sparse, a norm regularization would be somewhat ``informed'' and could aid in recovering it. We refer to \cite{lecue2017regularization,lecue2018regularization} for a further discussion about the interpretation of this type of regularization as the \textit{smoothing} of empirical error minimisation estimators or as sparse recovery methods.

Considering $\Psi$ as a norm is a strong restriction from an informed regularization point of view, especially when the target function is not ``sparse''. For instance, in this case $\Psi(h) = 0$ if, and only if, $h = 0$, so it is not possible to represent general prior information in the form $\Psi(\ustar) = 0$ when $\Psi$ is a norm. Therefore, to meaningfully represent prior information, such as equation \eqref{dif_eq}, more general regularization functions need to be considered and classical results about the statistical properties of norm-regularized estimators (see \cite{koltchinskii2011oracle} for examples) do not readily apply.

Error rates for more general regularization functions and classes of problems have been established based on the small-ball method \cite{lecue2017regularization,lecue2018regularization}. In this paper, we extend these results to physics-informed regularization functions of form $\Psi(h) = \lVert \mathscr{D}h - g \rVert_{L_{2}(\mu)}^{2}$ with $\ms{D}$ linear.

We note that the results of this paper do not depend on the differential nature of $\mathscr{D}$, as they hold for any linear operator $\mathscr{D}$ taking values in $L_{2}(\mu)$, so they cover prior information expressed as a linear equation $\mathscr{D}\ustar = g$ in general. In the particular case of linear classes of finite dimension, treated in Section \ref{SecFinDim}, the physics-informed estimator \eqref{informed_est} is least squares with a quadratic penalty on an affine map of the coefficients. This is a classical estimator, known in linear models as the mixed estimator under stochastic linear restrictions on the coefficients, proposed in \cite{theil1992pure}, see also \cite[Chapter~5]{rao2008linear}, and in inverse problems as Tikhonov regularization in general form \cite[Section~5.1]{hansen1998rank}. Its classical analyses consider a fixed design and homoscedastic or Gaussian noise, while our analysis in Section \ref{SecFinDim} addresses questions \textbf{(A)} and \textbf{(B)} under a random design and heavy-tailed noise. These results are obtained as a specialisation of our main theorem, Theorem \ref{main_theorem}.

\subsection{Small-ball method and complexity dependent error rates}
\label{Sec1.4}

Classical error rates in statistical learning depend on the existence of uniform concentration inequalities guaranteeing that, as $n$ increases, the empirical error $L_{n}(h)$ concentrates uniformly around the expected error $L(h) = \mathbb{E}(h(X) - Y)^{2}$ for $h \in \mathscr{H}$. Uniform concentration is known to hold, for example, when the loss function is bounded and $\mathscr{H}$ has finite VC dimension, and distribution-free \cite{vapnik1998statistical} and distribution-dependent \cite{bartlett2005local,bartlett2002rademacher,koltchinskii2006local} error rates have been developed in this case. We refer to \cite{concentration} for an overview of concentration inequalities.

In order to address instances where there is no uniform concentration, such as the simple case of least squares linear regression with Gaussian noise or general regression problems with heavy-tailed losses, \cite{mendelson2015learning} and \cite{mendelson2018learning} derived sharp bounds on the performance of empirical error minimisation estimators under the small-ball property (cf. Assumption \ref{small_ball}) when $\mathscr{H}$ is a convex class, for the quadratic and general sufficiently smooth convex loss functions, respectively. The small-ball property is a relatively weak assumption, and hence the bounds based on it are broadly applicable. We refer to the discussion in \cite{mendelson2015learning,mendelson2018learning} for more details.

The small-ball method made it possible to derive error rates for regularized estimators under the quadratic loss function when $\Psi$ is a norm and the function class is convex \cite{lecue2018regularization}, and for a slightly broader class of regularization functions \cite{lecue2017regularization}. In particular, \cite{lecue2017regularization} established that the error rate of a regularized estimator in that case is surprisingly close to that of empirical error minimisation in the constrained set $\mathscr{H}_{\Psi}$ (cf. \eqref{H_psi}), even when $\Psi(\hstar)$ is unknown. The error rates of \cite{lecue2017regularization} are called \textit{complexity dependent}, rather than \textit{sparsity-dependent}, since they depend on the complexity of the ``true model'' $\mathscr{H}_{\Psi}$ and not on the ``sparsity'' of $\hstar$ as in the case where $\Psi$ is a norm.

The results of \cite{lecue2017regularization} do not apply to physics-informed regularization in general, since they assume that $\Psi$ is nonnegative, even, convex and $\Psi(0) = 0$; there exists $\eta \geq 1$ such that, for all $f,h \in \mathscr{H}$, $\Psi(f + h) \leq \eta(\Psi(f) + \Psi(h))$; and $\Psi(\alpha h) \leq \alpha\Psi(h)$ for all $h \in \mathscr{H}$ and $0 \leq \alpha \leq 1$. Some of these are strong restrictions from the perspective of informed regularization, and clearly do not hold in general for $\Psi(h) = \lVert \mathscr{D}h - g \rVert_{L_{2}(\mu)}^{2}$. In \cite{scampicchio2025physics} it was assumed that $\ms{D}$ is elliptic and $g = 0$ because in this special case the regularization function $\Psi(h) = \lVert \mathscr{D}h \rVert_{L_{2}(\mu)}^{2}$ satisfies the assumptions of \cite{lecue2017regularization} and complexity dependent error rates could be deduced. In general cases, the results of \cite{lecue2017regularization} do not apply.

In this paper, we extend the complexity dependent bounds of \cite{lecue2017regularization} to physics-informed regularization when the operator $\mathscr{D}$ is linear, $g \neq 0$, and the small-ball assumption (cf. Assumption \ref{small_ball}) is satisfied on the convex hull of $\mathscr{H}$.

\subsection{Main contributions}

The main contributions of this paper are:
\begin{itemize}
	\item We extend the complexity dependent error bounds of \cite{lecue2017regularization} to informed regularization functions of form $\Psi(h) = \lVert \mathscr{D}h - g \rVert_{L_{2}(\mu)}^{2}$ with $\mathscr{D}$ linear when the convex hull of $\mathscr{H}$ satisfies a small-ball condition, and
	\item we specialise the results to linear classes of finite dimension, for which we obtain explicit error rates that address \textbf{(A)} and \textbf{(B)}, also under misspecified physics.
\end{itemize}

Regarding question \textbf{(A)}, we show that, in linear classes of finite dimension and when $\Psi(\hstar) = 0$, there is no price to pay, in terms of our bounds, for considering a physics-informed estimator rather than learning in $\mathscr{H}_{\Psi}$: when $\lambda$ is above an explicit threshold, the bound for the informed estimator $\hhat$ coincides with that for empirical error minimisation in $\ms{H}_{\Psi}$ obtained via the small-ball method \cite{mendelson2015learning}. When the physics is misspecified, the two bounds differ by a constant factor only, but $\lambda$ should be chosen in a range that depends on the size of the misspecification.

Regarding question \textbf{(B)}, we show that the error rates for physics-informed estimators are related to the \textit{complexity} of $\ms{H}_{\Psi}$, which can be much smaller than that of $\ms{H}$. In particular, when $\Psi(\hstar) = 0$, $\ms{H}_{\Psi} = \{h \in \mathscr{H}: \Psi(h) = 0\}$, so this complexity is associated with the solution space of $\ms{D}h = g$ in $\mathscr{H}$. When $g = 0$ and $\ms{H}$ is finite dimensional, $\mathscr{H}_{\Psi}$ usually has a dimension smaller than that of $\ms{H}$, so physics-informed penalization has the same effect on error rates as dimensionality reduction. In linear classes of finite dimension we show that this is precisely the case, as the penalty reduces the error rate from $\sqrt{q/n}$ to $\sqrt{k/n}$, in which $q$ and $k$ are the dimensions of $\mathscr{H}$ and of the solution space (Proposition \ref{cor_findim}). A similar fact has also been observed in the specific cases treated in \cite{doumeche2024physics,koshizuka2025understanding,scampicchio2025physics}.

\subsection{Paper structure}

In Section \ref{Sec1}, we present the reasoning that leads to the small-ball method in the context of informed penalization, and we state our main result. In Section \ref{SecImplications}, we specialise the result to linear classes of finite dimension and address questions \textbf{(A)} and \textbf{(B)}. In Section \ref{SecPoss}, we present extensions and limitations, and in Section \ref{SecFR} we give our final remarks. Proofs are given in Section \ref{SecProof}.

\section{Error Rates for Physics-informed Statistical Learning}
\label{Sec1}

In this section, we state the main result of this paper, which is proved in Section \ref{SecProof}. We start by formalising two assumptions about $\mu$ and $\mathscr{H}$ in Sections \ref{SecSB} and \ref{SecConc}: that a small-ball property is satisfied in the convex hull of $\mathscr{H}$, and that a concentration inequality holds for the excess empirical error of a reference target of $\mathscr{H}$ with respect to the target function of $\mathscr{F}$, respectively. Then, in Section \ref{SecCandidates} we discuss the candidate functions for $\hhat$ (cf. \eqref{informed_est}), in Section \ref{SecEL} we discuss how the excess loss can be bounded over the convex hull of $\ms{H}$ and in Section \ref{SecRC} we define the complexity measures that bounds derived from the small-ball method depend on. The content of Sections \ref{SecCandidates} to \ref{SecRC} is adapted from \cite{lecue2017regularization} to our context. Finally, in Section \ref{SecMR} we state our main result.

\subsection{Small-ball property of convex hull}
\label{SecSB}

The small-ball property allows the derivation of sharp bounds on the estimation error of empirical error minimisation procedures without relying on concentration inequalities as has been classically done in the theory of empirical processes (see \cite{koltchinskii2006local,koltchinskii2011oracle} and the references therein). We refer to \cite{koltchinskii2015bounding,mendelson2014remark,mendelson2015learning,mendelson2018learning} for more details about the so-called small-ball method.

We assume that $\ms{F} = \widebar{\cH}$, that is, the closure in $L_{2}(\mu)$ of the class of convex combinations of functions in $\ms{H}$, satisfies the following small-ball assumption\footnote{Actually, we only need the small-ball property to hold for $h - \fstar$ with $h \in \mathscr{F}$, but we keep this stronger assumption to be consistent with the small-ball literature. We note that proving the condition for a fixed $\fstar$ is usually not easier than for any pair $h$ and $f$.}. If $\mathscr{H}$ is closed and convex, then  $\ms{F} = \ms{H}$. Since $\Psi$ is defined for functions in $H^{s}(\mu)$, we also assume that $\ms{F} \subset H^{s}(\mu)$.

\begin{assumption}
	\label{small_ball}
	There exist constants $\kappa > 0$ and $0 < \epsilon \leq 1$ such that, for all $f,h \in \mathscr{F}$,
	\begin{equation*}
		Pr(\left|f(X) - h(X)\right| \geq \kappa \lVert f - h \rVert_{L_{2}(\mu)}) \geq \epsilon.
	\end{equation*}
\end{assumption}

As opposed to concentration inequalities, which rely heavily on $f(X)$ being bounded or not having heavy tails, uniformly for $f \in \mathscr{F}$, bounds based on Assumption \ref{small_ball} are more broadly applicable since this assumption is relatively weak (see the discussion in \cite{lecue2017regularization,lecue2018regularization,mendelson2015learning,mendelson2018learning}). For instance, it holds if there exists a constant $\kappa_{1} > 0$ such that $\lVert f - h \rVert_{L_{2}(\mu)} \leq \kappa_{1} \lVert f - h \rVert_{L_{1}(\mu)}$ or a constant $\kappa_{2} > 0$ such that, for $p > 2$, $\lVert f - h \rVert_{L_{p}(\mu)} \leq \kappa_{2} \lVert f - h \rVert_{L_{2}(\mu)}$ for all $f,h \in \mathscr{F}$ (see \cite[Lemma~4.1]{mendelson2015learning}).

\subsection{Concentration for a reference target function}
\label{SecConc}

When $\ms{H}$ is not convex, fix a $\gamma \geq 0$ and let\footnote{We assume that the infimum in \eqref{Hstar} and all other infimum and supremum we consider in this paper are measurable to avoid technicalities.}
\begin{linenomath}
	\begin{equation}
		\label{Hstar}
		\Hstar \coloneqq \Hstar(\gamma) = \left\{h \in \ms{H}: L(h) \leq \inf\limits_{h \in \ms{H}} L(h) + \gamma\right\}
	\end{equation}
\end{linenomath}
be the set of target functions of $\mathscr{H}$, recalling that $L(h) \coloneqq \mathbb{E}(h(X) - Y)^{2}$. If the infimum of $L$ is achieved in $\ms{H}$ we set $\gamma = 0$. Let $\fstar = \argminA_{f \in \ms{F}} \mathbb{E}(f(X) - Y)^{2}$ be the unique minimiser of the mean squared error in the convex class $\ms{F} = \widebar{\cH}$. From now on, we fix a reference target function $\hstar \in \Hstar$ and assume that the excess empirical error of $\hstar$ with respect to $\fstar$ concentrates in probability around its expectation.

\begin{assumption}
	\label{concentration}
	For each $\delta^{\prime} \in (0,1)$, there exists a sequence $\{d_{n}\}$ of non-negative numbers, which converges to zero as $n \to \infty$, such that
	\begin{align*}
		&Pr\left(L_{n}(\hstar) - L_{n}(\fstar) \leq L(\hstar) - L(\fstar) + d_{n}\right) \geq 1 - \delta^{\prime}
	\end{align*}
	for all $n \geq 1$.
\end{assumption}

Assumption \ref{concentration} holds under mild moment conditions, for example if $\hstar - \fstar \in L_{4}(\mu)$ and $\mathbb{E}[(\fstar(X) - Y)^{2} \mid X] \leq \bar{\sigma}^{2}$ almost surely for a constant $\bar{\sigma} > 0$. Indeed, since
\begin{equation*}
	(\hstar(X) - Y)^{2} - (\fstar(X) - Y)^{2} = (\hstar - \fstar)^{2}(X) + 2(\fstar(X) - Y)(\hstar - \fstar)(X),
\end{equation*}
the variance of the left-hand side is at most $2\lVert \hstar - \fstar \rVert_{L_{4}(\mu)}^{4} + 8\bar{\sigma}^{2}\lVert \hstar - \fstar \rVert_{L_{2}(\mu)}^{2}$, and it follows from Chebyshev's inequality that Assumption \ref{concentration} is satisfied, for every $\delta^{\prime} \in (0,1)$, with
\begin{equation*}
	d_{n} = \left[\frac{2\lVert \hstar - \fstar \rVert_{L_{4}(\mu)}^{4} + 8\bar{\sigma}^{2}\lVert \hstar - \fstar \rVert_{L_{2}(\mu)}^{2}}{n\delta^{\prime}}\right]^{1/2}.
\end{equation*}
We note that the moment conditions hold, for example, whenever $\ms{H}$ is bounded in $L_{\infty}(\mu)$, $\ustar$ is bounded and $e$ has finite variance.

\subsection{Candidates for $\hhat$}
\label{SecCandidates}

Recalling the definition of $L_{n,\Psi}$ in \eqref{PINN_problem}, observe that every minimiser $\hhat$ of $L_{n,\Psi}$ in $\ms{H}$ satisfies
\begin{linenomath}
	\begin{equation}
		\label{cond_min}
		L_{n}(\hhat) - L_{n}(\hstar) \leq \lambda(\Psi(\hstar) - \Psi(\hhat)).
	\end{equation}
\end{linenomath}

The analysis we carry out in this paper, which is adapted from \cite{lecue2017regularization}, relies on \eqref{cond_min} to exclude from $\ms{H}$ the functions that do not satisfy it as possible minimisers of $L_{n,\Psi}$. In particular, we aim to show that if $h \in \ms{H}$ and $\lVert h - \hstar \rVert_{L_{2}(\mu)}$ is not \textit{small}, then it does not satisfy \eqref{cond_min} and cannot be a minimiser of $L_{n,\Psi}$. This implies that $\hhat$ should be \textit{close} to a function in $\Hstar$.

\subsection{Excess loss in the convex hull of $\ms{H}$}
\label{SecEL}

Recall that $\fstar$ is the unique minimiser of the mean squared error in the convex class $\ms{F} = \widebar{\cH}$. Define $\xi = \fstar(X) - Y = \fstar(X) - \ustar(X) - e$ which reduces to $\xi = -e$ if $\ustar \in \mathscr{F}$. For $f \in \ms{F}$, define $\ell_{f}(X,Y) = (f(X) - Y)^{2} = (f(X) - \fstar(X) + \xi)^{2}$ as the squared error of $f$ on $(X,Y)$ and 
\begin{linenomath}
	\begin{align*}
		\mathcal{L}_{f}(X,Y) = \ell_{f}(X,Y) - \ell_{\fstar}(X,Y)
	\end{align*}
\end{linenomath}
as the excess loss of $f$ in $\ms{F}$. We denote by $P_{n}\mathcal{L}_{f}$ the empirical mean of $\mathcal{L}_{f}(X,Y)$ on the sample $S_{n}$. Since $P_{n}\mathcal{L}_{h} = L_{n}(h) - L_{n}(\fstar)$,
\begin{linenomath}
	\begin{align}
		\label{main_inequality}
		&L_{n}(h) - L_{n}(\hstar) = P_{n}\mathcal{L}_{h} - \left[L_{n}(\hstar) - L_{n}(\fstar)\right]
	\end{align}
\end{linenomath}
for all $h \in \ms{H} \subset \ms{F}$ and sample $S_{n}$, it follows from \eqref{cond_min} that we can disregard as candidates for $\hhat$ all $h \in \ms{H}$ such that
\begin{linenomath}
	\begin{align}
		\label{cond_in_F}
		&P_{n}\mathcal{L}_{h} > \lambda (\Psi(\hstar) - \Psi(h)) + L_{n}(\hstar) - L_{n}(\fstar).
	\end{align}
\end{linenomath}

If $\ms{H}$ is closed and convex, then $\hstar = \fstar$, and it follows from \eqref{cond_min} that, for any $\lambda > 0$, we can disregard as candidates for $\hhat$ all $h \in \ms{H}$ satisfying
\begin{equation}
	\label{cond_in_F2}
	P_{n}\mathcal{L}_{h} > \lambda\left(\Psi(\fstar) - \Psi(h)\right).
\end{equation}
In either case, to identify functions $f \in \ms{F} \supset \mathscr{H}$ that satisfy \eqref{cond_in_F} or \eqref{cond_in_F2} we need to obtain lower bounds for $P_{n}\mathcal{L}_{f}$.

Since $\ms{F}$ is a closed convex set, by the characterisation of the projection in convex Hilbert spaces,
\begin{linenomath}
	\begin{equation}
		\label{projection}
		\mathbb{E}\xi(f - \fstar)(X) = \mathbb{E}(\fstar(X) - Y)(f - \fstar)(X) \geq 0
	\end{equation}
\end{linenomath}
for all $f \in \ms{F}$ and hence
\begin{linenomath}
	\begin{align}
		\label{lb_PnL} \nonumber
		P_{n}\mathcal{L}_{f} &= \frac{1}{n} \sum_{i=1}^{n}(f - \fstar)^{2}(X_{i}) + \frac{2}{n}\sum_{i=1}^{n} \xi_{i}(f - \fstar)(X_{i})\\ \nonumber
		&\geq \frac{1}{n} \sum_{i=1}^{n}(f - \fstar)^{2}(X_{i}) + \frac{2}{n}\sum_{i=1}^{n} \left[\xi_{i}(f - \fstar)(X_{i}) - \mathbb{E}(\fstar(X) - Y)(f - \fstar)(X)\right]\\
		&\geq P_{n}\mathcal{Q}_{f-\fstar} - 2|P_{n}\mathcal{M}_{f - \fstar}|,
	\end{align}
\end{linenomath}
in which $\xi_{i} = \fstar(X_{i}) - Y_{i}$,
\begin{linenomath}
	\begin{align*}
		\mathcal{Q}_{f-\fstar}(X) = (f - \fstar)^{2}(X)\, , & &  \text{ and } & & \mathcal{M}_{f - \fstar}(X,\xi) = \xi(f - \fstar)(X) - \mathbb{E}\xi(f - \fstar).
	\end{align*}
\end{linenomath}

Therefore, to disregard \textit{bad} functions as candidates for $\hhat$ it is enough to obtain lower bounds for $P_{n} \mathcal{Q}_{h-\fstar}$ and upper bounds for $|P_{n}\mathcal{M}_{h - \fstar}|$ for $h \in \ms{H} \subset \mathscr{F}$ such that $\lVert h - \hstar \rVert_{L_{2}(\mu)}$ is \textit{large}, and compare them with $\lambda(\Psi(\hstar) - \Psi(h)) + L_{n}(\hstar) - L_{n}(\fstar)$ to attest that $P_{n}\mathcal{L}_{h}$ is greater than it. This implies that $\hhat$ should be \textit{close} to a function in $\Hstar$. The processes $P_{n} \mathcal{Q}_{h-\fstar}$ and $P_{n}\mathcal{M}_{h - \fstar}$ are called, respectively, the product and multiplier empirical processes. We refer to \cite{lecue2017regularization,lecue2018regularization} for more details on this kind of argument and to \cite{mendelson2016upper} for bounds on these processes in special cases.

We note that in the noiseless case, when $\xi = 0$ with probability one, $|\mc{M}_{h - \fstar}(X,\xi)| = 0$ and only bounds on the product process are necessary.

\subsection{Local Rademacher complexities}
\label{SecRC}

The results of this paper depend on the following bounds on Rademacher processes in $\ms{F}$. Denote by $D$ the unit ball in $L_{2}(\mu)$ and for $r > 0$ define $rD_{\fstar} = \{f \in L_{2}(\mu): \lVert f - \fstar \rVert_{L_{2}(\mu)} \leq r\}$.

\begin{definition}
	\label{def_rad}
	Given a class $\mathscr{G} \subset L_{2}(\mu)$ and $\tau > 0$, let
	\begin{linenomath}
		\begin{equation*}
			r_{Q}(\mathscr{G},\tau) \coloneqq r_{Q}(\mathscr{G},\tau,\fstar,n) = \inf\left\{r > 0 : \mathbb{E} \sup\limits_{f \in \mathscr{G} \cap r D_{\fstar}} \left|\frac{1}{n} \sum_{i=1}^{n} \varepsilon_{i}(f - \fstar)(X_{i})\right| \leq \tau r\right\}
		\end{equation*}
	\end{linenomath}
	in which $\varepsilon_{1},\dots,\varepsilon_{n}$ are independent, symmetric, $\{-1,+1\}$-valued random variables independent of $S_{n}$, and the expectation is taken on the variables $\varepsilon_{i}$ and $S_{n}$. Set
	\begin{linenomath}
		\begin{equation*}
			\phi_{n}(\mathscr{G},\fstar,s) = \sup\limits_{f \in \mathscr{G} \cap sD_{\fstar}} \left|\frac{1}{\sqrt{n}} \sum_{i=1}^{n} \varepsilon_{i}\xi_{i}(f - \fstar)(X_{i})\right|
		\end{equation*}
	\end{linenomath}
	and define
	\begin{linenomath}
		\begin{equation*}
			r_{M}(\mathscr{G},\tau,\delta) \coloneqq r_{M}(\mathscr{G},\tau,\delta,\fstar,n) = \inf\left\{s > 0: Pr\left(\phi_{n}(\mathscr{G},\fstar,s) \leq \tau s^{2} \sqrt{n}\right) \geq 1 - \delta\right\}.
		\end{equation*}
	\end{linenomath}
\end{definition}

On the one hand, the parameter $r_{Q}(\mathscr{G},\tau)$ represents the rate at which the \textit{local} Rademacher averages of $(f - \fstar)(X)$ in the balls of $\ms{G}$ centred at $\fstar$ decay in expectation as a function of $n$. By the usual symmetrisation arguments of empirical processes (see \cite[Section~11.3]{concentration}), under the small-ball assumption, $r_{Q}(\mathscr{G},\tau)$ is related to the expectation of the product process. On the other hand, the parameter $r_{M}(\mathscr{G},\tau,\delta)$ represents the rate at which the Rademacher averages of $\xi(f - \fstar)(X)$ in the balls of $\ms{G}$ centred at $\fstar$ decay with high probability as a function of $n$. Again by the usual symmetrisation arguments, $r_{M}(\mathscr{G},\tau,\delta)$ is related to the multiplier process.

These quantities control two possible sources of \textit{error} in statistical learning: the complexity of the version space and the data noise. The version space is the random set of functions $f \in \ms{F}$ which agree with $\fstar$ on the points $X_{i}$ in the sample: $f(X_{i}) = \fstar(X_{i})$. If the $L_{2}(\mu)$ diameter of the version space is great, that is, if there are functions which agree with $\fstar$ on the sample but are far from $\fstar$ in $L_{2}(\mu)$, then large mistakes can occur when learning a function from this space. The parameter $r_{Q}(\mathscr{F},\tau)$ controls the \textit{local size} of the version space and, if it is generally small and if there is no noise ($\xi = 0$), then large mistakes should be unlikely.

Conversely, when the noise is \textit{high}, it should be the predominant source of error. As it happens, this error is captured by the correlation between $\varepsilon\xi$ and $(f - \fstar)(X)$ as appearing in the Rademacher average in $r_{M}(\mathscr{G},\tau,\delta)$. When $Y$ is bounded with probability one, and $\ms{F}$ is bounded in $L_{\infty}(\mu)$, the usual contraction arguments allow to bound expectation and probabilities of the multiplier process by that of the product process (see \cite[Section~11.3]{concentration}). In this case, the correlation of $(f - \fstar)(X)$ with $\varepsilon\xi$ is replaced by the correlation with only the \textit{generic error} $\varepsilon$, which can lead to worse rates in the bounded case. In the unbounded case, the noise must be controlled in $L_{2}(\mu)$ somehow. We refer to the discussion in \cite[Section~2.2]{mendelson2015learning} for more details about the parameters $r_{Q}(\mathscr{G},\tau)$ and $r_{M}(\mathscr{G},\tau,\delta)$ and their usefulness for bounding error rates.

Moving on, for $\rho \geq 0$ let
\begin{linenomath}
	\begin{equation}
		\label{Frho}
		\mathscr{F}(\rho,\Psi) = \{f \in \mathscr{F}: \Psi(f) - \Psi(\hstar)\leq \rho\}
	\end{equation}
\end{linenomath}
be the subset of functions in $\ms{F}$ such that $\Psi(f)$ is either lesser than $\Psi(\hstar)$, or is greater than it by at most $\rho$. Denote
\begin{linenomath}
	\begin{align}
		\label{def_r}
		r_{M}(\rho) = r_{M}\left(\mathscr{F}(\rho,\Psi),\frac{\kappa^{2}\epsilon}{80},\frac{\delta}{4}\right) \text{ and } r_{Q}(\rho) = r_{Q}\left(\mathscr{F}(\rho,\Psi),\frac{\kappa\epsilon}{32}\right)
	\end{align}
\end{linenomath}
for a fixed $\delta \in (0,1)$, in which $\kappa$ and $\epsilon$ are constants such that Assumption \ref{small_ball} holds for $\ms{F}$. Note that $r_{M}(\rho)$ and $r_{Q}(\rho)$ are non-decreasing in $\rho$, since $\mathscr{F}(\rho,\Psi)$ increases with $\rho$ and enlarging the class enlarges the suprema in Definition \ref{def_rad}. Let $r(\cdot)$ be a function, which may depend on $\fstar$ and other parameters such as $\delta,\kappa$ and $\epsilon$, satisfying $r(\rho) \geq \max\{r_{M}(\rho),r_{Q}(\rho)\}$ for all $\rho \geq 0$.

\subsection{Statement of main result}
\label{SecMR}

Recalling the reference target function $\hstar \in \Hstar$, define the error gap and the penalization gap between $\hstar$ and the target $\fstar$ of $\ms{F}$ as
\begin{linenomath}
	\begin{equation}
		\label{gaps}
		\Delta_{L} \coloneqq L(\hstar) - L(\fstar) \geq 0 \qquad \text{ and } \qquad \Delta_{\Psi} \coloneqq \Psi(\hstar) - \Psi(\fstar),
	\end{equation}
\end{linenomath}
and let $\Psi_{\ms{H}} \coloneqq \inf_{h \in \ms{H}} \Psi(h)$. When $\fstar \in \ms{H}$, for instance when $\mathscr{H}$ is closed and convex, $\Delta_{L} = \Delta_{\Psi} = 0$. The main result of this paper states that, with high probability, $\Psi(\hhat) - \Psi(\hstar) \leq \rho$, and $\lVert \hhat - \fstar \rVert_{L_{2}(\mu)}$ is bounded in terms of $r(\rho)$ and of the gaps \eqref{gaps}.

\begin{theorem}
	\label{main_theorem}
	Let $\ms{H}$ be a class of functions such that $\ms{F} = \widebar{\cH}$ satisfies Assumption \ref{small_ball} with constants $\kappa$ and $\epsilon$. Fix $\gamma \geq 0$, $\delta^\prime \in (0,1)$ and $\hstar \in \Hstar(\gamma)$ satisfying Assumption \ref{concentration} with a sequence $\{d_{n}\}$, and fix $\delta \in (0,1)$, $\rho > \max\{0,-\Delta_{\Psi}\}$ and $\lambda > 0$ such that
	\begin{linenomath}
		\begin{equation}
			\label{cond_lambda}
			\lambda\left(\rho + \min\{\Delta_{\Psi},0\}\right) > \frac{\kappa^{2}\epsilon}{32} \, r^{2}(\rho) + \Delta_{L} + d_{n}.
		\end{equation}
	\end{linenomath}
	Then, with probability at least $1 - \delta - \delta^{\prime} - 2\exp(-n\epsilon^{2}/2)$,
	\begin{linenomath}
		\begin{equation*}
			\hhat \in \{h \in \mathscr{H}: \Psi(h) - \Psi(\hstar) \leq \rho\}
		\end{equation*}
	\end{linenomath}
	and
	\begin{linenomath}
		\begin{align}
			\label{theorem_improve_eq}
			\lVert \hhat - \fstar \rVert_{L_{2}(\mu)} \leq \max\left\{r(\rho),\left[\frac{32}{\kappa^{2}\epsilon}\left(\lambda\left(\Psi(\hstar) - \Psi_{\ms{H}}\right) + \Delta_{L} + d_{n}\right)\right]^{1/2}\right\}.
		\end{align}
	\end{linenomath}
	In particular, if $\fstar \in \ms{H}$, then one may take $\gamma = 0$, $\hstar = \fstar$, $d_{n} = 0$ and $\delta^{\prime} = 0$, in which case Assumption \ref{concentration} is trivially satisfied, $\Delta_{L} = \Delta_{\Psi} = 0$ and condition \eqref{cond_lambda} reduces to $\lambda\rho > \kappa^{2}\epsilon \, r^{2}(\rho)/32$.
\end{theorem}

When $\fstar \in \mathscr{H}$, condition \eqref{cond_lambda} reduces to $\lambda\rho > \kappa^{2}\epsilon r^{2}(\rho)/32$ and, if the minimum of $\Psi$ in $\mathscr{H}$ is achieved at $\fstar$, that is, $\Psi(\fstar) = \Psi_{\ms{H}}$, then the bound \eqref{theorem_improve_eq} reduces to $r(\rho)$. In this case, the application of Theorem \ref{main_theorem} proceeds in two steps. The first is to bound $r_{Q}(\rho)$ and $r_{M}(\rho)$, that is, to control the product and multiplier processes on the sub-level sets $\ms{F}(\rho,\Psi)$, obtaining an admissible $r(\rho)$ as an explicit function of $\rho$ and the sample size $n$. The second step is, for a given $\lambda$, to take $\rho = \rho_{\lambda} > 0$ as small as possible satisfying \eqref{cond_lambda}, so that, with high probability, $\Psi(\hhat) \leq \Psi(\fstar) + \rho_{\lambda}$ and $\lVert \hhat - \fstar \rVert_{L_{2}(\mu)} \leq r(\rho_{\lambda})$.

On the other hand, when $\fstar \in \ms{H}$ but $\Psi(\fstar) > \Psi_{\ms{H}}$, the second term in \eqref{theorem_improve_eq} vanishes only if $\lambda = \lambda_{n} \to 0$ as $n$ increases, which by \eqref{cond_lambda} requires a sequence $\rho_{n}$ with $r^{2}(\rho_{n})/\rho_{n} \to 0$ and $r(\rho_{n}) \to 0$. This sample-size-dependent $\lambda$ is part of the price to pay for a misspecified physics in $\mathscr{H}$, when its target function is not among the minimisers of $\Psi$ in $\mathscr{H}$.

Finally, when $\fstar \notin \mathscr{H}$, a $\rho > 0$ must be fixed and $\lambda$ chosen according to \eqref{cond_lambda}, whose right-hand side does not vanish as $n \to \infty$. By bounding $r(\rho)$, one can then estimate the minimum sample size $n$ for which the right-hand side of \eqref{theorem_improve_eq} is dominated by the irreducible errors $\Delta_{L}$ and $\lambda(\Psi(\hstar) - \Psi_{\ms{H}})$, the latter vanishing when $\hstar$ minimises $\Psi$ in $\ms{H}$.

\begin{remark}	
	\label{rem_gen}
	We note that Theorem \ref{main_theorem} holds for any real-valued function $\Psi$ on $\ms{F}$ that is bounded from below on $\ms{H}$, convex, in the sense that $\Psi(\alpha h + (1-\alpha)f) \leq \alpha\Psi(h) + (1-\alpha)\Psi(f)$ for $0 < \alpha < 1$ and all $h,f \in \mathscr{F}$, and such that $t \mapsto \Psi(\fstar + t(f - \fstar))$ is continuous on $[0,1]$ for all $f \in \ms{F}$. Indeed, the linearity of $\mathscr{D}$ is used in the proof only to show that $\ms{F}(\rho,\Psi)$ is convex, and to deduce identity \eqref{Rapproxi}, of which only the inequality $\Psi(f) - \Psi(\fstar) \geq R\left[\Psi(h) - \Psi(\fstar)\right]$, which follows from the convexity of $\Psi$, is used. These are weaker conditions than those of \cite[Theorem~1.9]{lecue2017regularization}, which assumes that $\Psi$ has the properties described in Section \ref{Sec1.4}. Also, the effective space in that case is $\{f \in \ms{H}: \Psi(f - \fstar) \leq \rho\}$ instead of $\ms{F}(\rho,\Psi)$ in \eqref{Frho}.	
\end{remark}

\section{Statistical effect of physics-informed penalization}
\label{SecImplications}

In order to understand the statistical effect of physics-informed penalization, we will compare the bounds of Theorem \ref{main_theorem} with those obtained via the small-ball method by \cite{mendelson2015learning} for plain empirical error minimisation on the restricted class $\ms{H}_{\Psi}$ (cf. \eqref{H_psi}) and on the whole set $\ms{H}$ in particular cases. These comparisons address questions \textbf{(A)} and \textbf{(B)}, respectively, when the target hypothesis $\fstar$ of $\ms{H}$ satisfies the prior information \eqref{dif_eq}. We assume that $\ms{H}$ is closed and convex, to be compatible with the assumptions of \cite{mendelson2015learning}. We state the main result of \cite{mendelson2015learning} as stated in \cite[Theorem~1.6]{lecue2017regularization} for reference.

\begin{theorem}[Theorem 1.6 of \cite{lecue2017regularization}]
	\label{theorem_SN}
	Let $\ms{G} \subset L_{2}(\mu)$ be a closed, convex class of functions that satisfies Assumption \ref{small_ball} with constants $\kappa$ and $\epsilon$. For every $0 < \delta < 1$, with probability at least $1 - \delta - 2\exp\left(-n\epsilon^{2}/2\right)$, it holds
	\begin{align}
		\label{bound_sb}
		\lVert \hat{g}_{n} - \fstar \rVert_{L_{2}(\mu)} \leq \max\left\{r_{M}(\ms{G},\kappa^{2}\epsilon/80,\delta/4),r_{Q}(\ms{G},\kappa\epsilon/32)\right\}
	\end{align}
	in which $\hat{g}_{n}$ is a minimiser of $L_{n}(h)$ in $\ms{G}$ and $\fstar$ is the minimiser of $L$ in $\ms{G}$.
\end{theorem}

Comparing Theorems \ref{main_theorem} and \ref{theorem_SN} we see that, since $r_{Q}(\rho)$ and $r_{M}(\rho)$ are non-decreasing in $\rho$ (cf. \eqref{def_r}) and $\ms{F}(0,\Psi) = \ms{H}_{\Psi}$ (cf. \eqref{H_psi}), $r(0)$ can be taken as the rate given in \eqref{bound_sb} of empirical error minimisation under the hard constraint. Then, in the case in which $\fstar \in \mathscr{H}$ and $\Psi(\fstar) = \Psi_{\ms{H}}$, question \textbf{(A)} is addressed by determining for which values of $\lambda$ one has $r(\rho_{\lambda}) \approx r(0)$ with $\rho_{\lambda} > 0$ as in Section \ref{SecMR}, and \textbf{(B)} is addressed by comparing $r(\rho_{\lambda})$ with the respective bounds on the product and multiplier processes on the whole space $\ms{H}$ that appear on the right-hand side of \eqref{bound_sb}. In the next section, we address these questions for finite-dimensional classes.

\subsection{Finite-dimensional classes}
\label{SecFinDim}

Let $\ms{H} \subset H^{s}(\mu)$ be a linear space of finite dimension $q$, so $\ms{F} = \ms{H}$ and $\fstar \in \ms{H}$. Since all norms on a finite-dimensional space are equivalent, there is $\kappa_{1} < \infty$ such that $\lVert v \rVert_{L_{2}(\mu)} \leq \kappa_{1}\lVert v \rVert_{L_{1}(\mu)}$ for all $v \in \ms{H}$, hence Assumption \ref{small_ball} holds on $\ms{H}$, with constants $\kappa$ and $\epsilon$ that depend in particular on its dimension and on $\mu$ \cite[Lemma~4.1]{mendelson2015learning}. We fix such constants throughout this section.

Decompose $\ms{H} = W \oplus U$ with $W \coloneqq \{v \in \ms{H}: \mathscr{D}v = 0\}$ and $U$ its orthogonal complement in $\ms{H}$ with respect to the inner product of $L_{2}(\mu)$. Denote $k \coloneqq \dim W$, so $\dim U = q - k$, and assume that $k < q$, since otherwise $\Psi$ is constant on $\ms{H}$ and $\hhat$ is a minimiser of $L_{n}$. Let $T_{\ms{D}}: \mathscr{D}(\ms{H}) \to U$ be the inverse of the restriction of $\mathscr{D}$ to $U$, which is injective since $W \cap U = \{0\}$, $e_{1},\dots,e_{q-k}$ be an orthonormal basis of $\mathscr{D}(\ms{H})$ in $L_{2}(\mu)$, and
\begin{equation}
	\label{cD}
	c_{\ms{D}} \coloneqq \sqrt{q - k} \, \Big(\sum_{j = 1}^{q - k} \lVert T_{\ms{D}}e_{j} \rVert_{L_{2}(\mu)}^{2}\Big)^{-1/2},
\end{equation}
that is, $\sqrt{q - k}$ divided by the Hilbert-Schmidt norm of $T_{\ms{D}}$, which does not depend on the basis and is well defined since $\lVert \mathscr{D}u \rVert_{L_{2}(\mu)} > 0$ for all $u \in U, u \neq 0$.

For simplicity, we also assume that $\Psi(\fstar) = 0$, but the results hold with heavier notation if $\Psi(\fstar) = \Psi_{\mathscr{H}} > 0$. Under this assumption, the constrained class \eqref{H_psi} is given by
\begin{linenomath}
	\begin{equation}
		\label{Hpsi_findim}
		\ms{H}_{\Psi} = \ms{F}(0,\Psi) = \{h \in \ms{H}: \mathscr{D}h = g\} = \ms{H} \cap (\fstar + W).
	\end{equation}
\end{linenomath}
Finally, we assume that $\mathbb{E}[\xi^{2} \mid X] \leq \bar{\sigma}^{2}$ almost surely for a constant $\bar{\sigma} > 0$, recalling that $\xi = \fstar(X) - Y$, and that $n \geq 1024 \, q/(\kappa\epsilon)^{2}$, and we denote
\begin{linenomath}
	\begin{equation}
		\label{r_j}
		r_{j} \coloneqq C_{1}\bar{\sigma}\sqrt{\frac{j}{\delta n}}, \qquad C_{1} \coloneqq \frac{160\sqrt{8}}{\kappa^{2}\epsilon}, \qquad j \in \{k, \, q - k, \, q\}.
	\end{equation}
\end{linenomath}
See Remark \ref{sub_Gaussian} about the dependence of $r_{j}$ on $\delta$. The condition on $n$ ensures that the product process is controlled at every scale, so that $r_{Q}$ vanishes and the bound depends only on $r_{M}(\rho)$. The next lemma bounds $r(\rho)$ in this setting.

\begin{lemma}
	\label{lemma_findim}
	Let $\ms{H} \subset H^{s}(\mu)$ be a linear space of finite dimension $q$ and let $\kappa$ and $\epsilon$ be constants for which Assumption~\ref{small_ball} holds on $\ms{H}$. Let $W \coloneqq \{v \in \ms{H} : \mathscr{D}v = 0\}$ with $k \coloneqq \dim W < q$, and let $c_{\ms{D}}$ be given by \eqref{cD}. Assume that $\Psi(\fstar) = 0$, that $\mathbb{E}[\xi^{2} \mid X] \leq \bar{\sigma}^{2}$ almost surely for some $\bar{\sigma} > 0$, and that $n \geq 1024\, q/(\kappa\epsilon)^{2}$, and let $r_{j}$ be given by \eqref{r_j}. Then, for every $\rho \geq 0$, $r_{Q}(\rho) = 0$ and
	\begin{linenomath}
		\begin{equation}
			\label{r_V}
			r_{M}(\rho) \leq r_{\ms{H}}(\rho) \coloneqq \max\left\{r_{k}, \ \left(r_{q-k}\frac{\sqrt{\rho}}{c_{\ms{D}}}\right)^{1/2}\right\}.
		\end{equation}
	\end{linenomath}
	In particular, $r_{\ms{H}}$ is an admissible choice of $r(\cdot)$ in Theorem \ref{main_theorem}, and $r_{\ms{H}}(0) = r_{k}$.
\end{lemma}

The next proposition specialises Theorem \ref{main_theorem} to finite-dimensional classes.

\begin{proposition}
	\label{cor_findim}
	Under the assumptions of Lemma \ref{lemma_findim}, for $\lambda > 0$ let
	\begin{linenomath}
		\begin{equation}
			\label{s_lambda}
			s_{\lambda} \coloneqq \max\left\{r_{k}, \ \frac{C_{2}}{c_{\ms{D}}}\bar{\sigma}\sqrt{\frac{q - k}{\delta n \lambda}}\right\}, \qquad C_{2} \coloneqq \frac{80\sqrt{2}}{\kappa\sqrt{\epsilon}},
		\end{equation}
	\end{linenomath}
	and $\rho_{\lambda} \coloneqq \kappa^{2}\epsilon \, s_{\lambda}^{2}/(16\lambda)$. Then, with probability at least $1 - \delta - 2\exp(-n\epsilon^{2}/2)$,
	\begin{linenomath}
		\begin{equation}
			\label{cor_bound}
			\Psi(\hhat) \leq \rho_{\lambda} \qquad \text{ and } \qquad \lVert \hhat - \fstar \rVert_{L_{2}(\mu)} \leq s_{\lambda}.
		\end{equation}
	\end{linenomath}
	In particular, if
	\begin{linenomath}
		\begin{equation}
			\label{lambda_V}
			\lambda \geq \lambda_{\ms{H}} \coloneqq \frac{\kappa^{2}\epsilon}{16} \cdot \frac{q - k}{c_{\ms{D}}^{2} \, k},
		\end{equation}
	\end{linenomath}
	then $\lVert \hhat - \fstar \rVert_{L_{2}(\mu)} \leq r_{k}$ with the same probability.
\end{proposition}

Proposition \ref{cor_findim} and Theorem \ref{theorem_SN} allow us to answer questions \textbf{(A)} and \textbf{(B)} for general finite-dimensional classes. Indeed, since $\ms{H}_{\Psi} = \ms{H} \cap (\fstar + W)$ is closed and convex, and $\fstar$ minimises $L$ over it, we can apply Theorem \ref{theorem_SN}. From Lemma \ref{lemma_findim} with $\rho = 0$ it follows that $r_{M}(\ms{H}_{\Psi},\kappa^{2}\epsilon/80,\delta/4) \leq r_{k}$ and $r_{Q}(\ms{H}_{\Psi},\kappa\epsilon/32) = 0$ for $n \geq 1024 \, q/(\kappa\epsilon)^{2}$, hence a minimiser $\hat{h}_{n,\ms{H}_{\Psi}}$ of $L_{n}$ over $\ms{H}_{\Psi}$, that is, empirical error minimisation under the hard constraint, satisfies $\lVert \hat{h}_{n,\ms{H}_{\Psi}} - \fstar \rVert_{L_{2}(\mu)} \leq r_{k}$ with probability at least $1 - \delta - 2\exp(-n\epsilon^{2}/2)$. By Proposition \ref{cor_findim}, the physics-informed estimator satisfies the same bound with the same probability for every $\lambda \geq \lambda_{\ms{H}}$, hence in this setting there is \textit{no price} to pay for replacing the hard constraint by the soft penalty as long as $\lambda$ is large enough, which answers \textbf{(A)}.

Regarding \textbf{(B)}, by Remark \ref{remarkH}, if $n \geq 1024\, q/(\kappa\epsilon)^{2}$, then $r_{M}(\ms{H},\kappa^{2}\epsilon/80,\delta/4) \leq r_{q}$ and $r_{Q}(\ms{H},\kappa\epsilon/32) = 0$, so Theorem \ref{theorem_SN} bounds the error of a minimiser $\hat{h}_{n}$ of $L_{n}$ over $\ms{H}$ by $r_{q}$, with the same probability as Proposition \ref{cor_findim}. Therefore, the physics-informed penalty reduces the error rate from $r_{q}$ to $r_{k}$, that is, by the factor $\sqrt{k/q}$, which is the effect of reducing the dimension of the class from $q$ to $k$, that of the kernel of $\ms{D}$.

\begin{remark}
	Since $q$ is fixed and the same constants $\kappa$ and $\epsilon$ appear in the bounds of learning in $\ms{H}$, in $\ms{H}_{\Psi}$ and via PISL, their dependence on $q$ and $\mu$ is not relevant to address \textbf{(A)} and \textbf{(B)}. They become important when the dimension increases with sample size, a scenario studied in the small-ball literature \cite{lecue2017regularization,lecue2018regularization}, but not considered in this paper.
\end{remark}

\subsubsection{Misspecified physics}

We now drop the assumption that $\Psi(\fstar) = \Psi_{\ms{H}}$, so the target $\fstar$ of $\ms{H}$ does not minimise the penalty term in $\ms{H}$, and $\Delta_{\ms{H}} \coloneqq \Psi(\fstar) - \Psi_{\ms{H}} > 0$ measures the misspecification of the physics relative to $\ms{H}$. Theorem \ref{main_theorem} yields, for every admissible $r(\cdot)$, every $\rho > 0$ and every $\lambda$ such that $\lambda\rho > \kappa^{2}\epsilon \, r^{2}(\rho)/32$, with probability at least $1 - \delta - 2\exp(-n\epsilon^{2}/2)$,
\begin{linenomath}
	\begin{equation}
		\label{miss_bound}
		\Psi(\hhat) \leq \Psi(\fstar) + \rho \qquad \text{ and } \qquad \lVert \hhat - \fstar \rVert_{L_{2}(\mu)} \leq \max\left\{r(\rho), \left[\frac{32\lambda\Delta_{\ms{H}}}{\kappa^{2}\epsilon}\right]^{1/2}\right\}.
	\end{equation}
\end{linenomath}
The second term in the maximum is the bias introduced by penalising when the target does not minimise the equation residual in $\ms{H}$, and it imposes an upper constraint on $\lambda$, since it does not exceed $r(\rho)$ if, and only if, $\lambda \leq \kappa^{2}\epsilon \, r^{2}(\rho)/(32\Delta_{\ms{H}})$, while Theorem \ref{main_theorem} requires $\lambda > \kappa^{2}\epsilon \, r^{2}(\rho)/(32\rho)$. The two constraints are compatible if, and only if, $\rho > \Delta_{\ms{H}}$.

Two choices of $r(\cdot)$ are available. On the one hand, Lemma \ref{lemma_findim2} in Section \ref{SecProof} proves that $r_{\ms{H}}(\rho) \coloneqq \max\{r_{k}, (r_{q-k}t_{\rho})^{1/2}\}$, with $t_{\rho} \coloneqq (\sqrt{\Delta_{\ms{H}} + \rho} + \sqrt{\Delta_{\ms{H}}})/c_{\ms{D}}$, is admissible. On the other hand, $r_{q}$ is admissible for every $\Delta_{\ms{H}}$, since $r_{M}(\rho) \leq r_{M}(\ms{H},\kappa^{2}\epsilon/80,\delta/4) \leq r_{q}$ and $r_{Q}(\rho) \leq r_{Q}(\ms{H},\kappa\epsilon/32) = 0$ for every $\rho \geq 0$ by Remark \ref{remarkH}. Hence, taking $r(\rho) = \min\{r_{q}, r_{\ms{H}}(\rho)\}$ and $\rho = (1 + \varsigma)\Delta_{\ms{H}}$ for any $\varsigma > 0$, for which $t_{\rho} = (1 + \sqrt{2 + \varsigma})\sqrt{\Delta_{\ms{H}}}/c_{\ms{D}}$, and denoting
\begin{linenomath}
	\begin{equation}
		\label{miss_rate}
		\bar{r} \coloneqq \min\left\{r_{q}, \ \max\left\{r_{k}, \ \left((1 + \sqrt{2 + \varsigma}) \, r_{q-k}\frac{\sqrt{\Delta_{\ms{H}}}}{c_{\ms{D}}}\right)^{1/2}\right\}\right\} \qquad \text{ and } \qquad \bar{\lambda} \coloneqq \frac{\kappa^{2}\epsilon \, \bar{r}^{2}}{32\Delta_{\ms{H}}},
	\end{equation}
\end{linenomath}
it follows from \eqref{miss_bound} that, with probability at least $1 - \delta - 2\exp(-n\epsilon^{2}/2)$,
\begin{linenomath}
	\begin{equation}
		\label{miss_final}
		\Psi(\hhat) \leq \Psi(\fstar) + (1 + \varsigma)\Delta_{\ms{H}} \qquad \text{ and } \qquad \lVert \hhat - \fstar \rVert_{L_{2}(\mu)} \leq \bar{r}
	\end{equation}
\end{linenomath}
for every $\bar{\lambda}/(1 + \varsigma) < \lambda \leq \bar{\lambda}$.

The rate $\bar{r}$ interpolates between the rate $r_{k}$ of the well-specified case and the rate $r_{q}$ of the empirical error minimiser over $\ms{H}$. In particular, let
\begin{linenomath}
	\begin{equation}
		\label{rho_V}
		\rho_{\ms{H}} \coloneqq \frac{c_{\ms{D}}^{2} \, k}{4(q - k)} \, r_{k}^{2} = \frac{C_{1}^{2} \, c_{\ms{D}}^{2} \, k^{2}}{4(q - k)} \cdot \frac{\bar{\sigma}^{2}}{\delta n},
	\end{equation}
\end{linenomath}
so that $r_{k}^{2}/r_{q-k} = 2\sqrt{\rho_{\ms{H}}}/c_{\ms{D}}$. If $\Delta_{\ms{H}} \leq 4\rho_{\ms{H}}/(1 + \sqrt{2 + \varsigma})^{2}$, then $\bar{r} = r_{k}$ and the rate of the well-specified case survives a misspecification of order $\rho_{\ms{H}}$. For greater $\Delta_{\ms{H}}$, the bound improves on $r_{q}$ if, and only if, $\sqrt{\Delta_{\ms{H}}} < c_{\ms{D}} \, r_{q}^{2}/((1 + \sqrt{2 + \varsigma}) \, r_{q-k})$. Otherwise $\bar{r} = r_{q}$, so the informed estimator is not worse than empirical error minimisation over $\ms{H}$, but the physics no longer helps.

In summary, the penalty is beneficial compared to empirical error minimisation when the excess residual $\Delta_{\ms{H}}$ of the equation at the target is below $\sim r_{q}^{4}/r_{q-k}^{2} = \mathcal{O}(1/n)$. Therefore, for a fixed $\Delta_{\ms{H}} > 0$, when the sample size increases, one eventually has $\bar{r} = r_{q}$, and $\lambda = \lambda_{n} \leq \bar{\lambda} \to 0$ is needed for consistency. For a finite sample, however, an approximately correct physics may still improve on empirical error minimisation over $\ms{H}$.

Regarding \textbf{(A)}, applying Theorem \ref{theorem_SN} to $\ms{H}_{\Psi} = \ms{F}(0,\Psi)$ (cf. \eqref{H_psi}), which is closed, convex and contains $\fstar$, we obtain the bound $\min\{r_{q}, r_{\ms{H}}(0)\}$. Since $\bar{r} \leq [(1 + \sqrt{2 + \varsigma})/2]^{1/2}\min\{r_{q}, r_{\ms{H}}(0)\}$, the bounds for the soft penalty and for this hard constraint differ by a constant factor only, for every $\Delta_{\ms{H}} > 0$. The price to pay is in the choice of $\lambda$, which is now also constrained from above by $\bar{\lambda}$, a quantity that depends on the unknown $\Delta_{\ms{H}}$, while under a well-specified physics every $\lambda \geq \lambda_{\ms{H}}$ attains the hard-constraint rate.

\subsection{Examples}

As an example, let $\ms{H}$ be the set of polynomials $h: [0,1] \times [0,T] \to \mathbb{R}$ of degree at most $p \geq 2$ in each variable. We write the functions in $\ms{H}$ as
\begin{align*}
	h_{a}(x,t) = \sum_{i,j = 0}^{p} a_{i,j} \, x^{i} \, t^{j}
\end{align*}
for coefficients $a \coloneqq (a_{i,j}) \in \mathbb{R}^{(p+1)^{2}}$. Let $\ms{D}$ be the heat operator, so that
\begin{align*}
	\ms{D}h_{a} = \frac{\partial h_{a}}{\partial t} - \frac{\partial^{2} h_{a}}{\partial x^{2}} 
	&= \sum_{i = 0}^{p} \sum_{j = 0}^{p} (a_{i,j+1} \, (j + 1) - a_{i + 2,j} \, (i + 1) \, (i + 2)) \, x^{i}t^{j}
\end{align*}
with the convention that $a_{i,j} = 0$ if $\max\{i,j\} > p$. In this case, with $g = 0$,
\begin{align*}
	\ms{H}_{\Psi} = \left\{h_{a} \in \ms{H}: a_{i,j+1} \, (j + 1) - a_{i + 2,j} \, (i + 1) \, (i + 2) = 0, \, \forall i,j\right\}
\end{align*}
which is a subspace of dimension $k = p + 1$, compared with the dimension $q = k^2 = (p + 1)^{2}$ of $\ms{H}$.

As a second example, with a hyperbolic equation, let $\ms{H}$ be the set of trigonometric polynomials $h: [0,1]^{2} \to \mathbb{R}$ of degree at most $p \geq 1$ in each variable, that is, the linear span of the functions $\cos(2\pi(ix + jt))$ and $\sin(2\pi(ix + jt))$ for integers $|i|,|j| \leq p$, which has dimension $q = (2p + 1)^{2}$, and let $\ms{D}$ be the transport operator, so that $\ms{D}h = \frac{\partial h}{\partial t} + \frac{\partial h}{\partial x}$. Since
\begin{align*}
	& \ms{D}\cos(2\pi(ix + jt)) = -2\pi(i + j)\sin(2\pi(ix + jt)) \quad \text{ and } \\
	& \ms{D}\sin(2\pi(ix + jt)) = 2\pi(i + j)\cos(2\pi(ix + jt)),
\end{align*}
it follows that
\begin{align*}
	\ms{H}_{\Psi} = \text{span}\left\{\cos(2\pi i(x - t)), \, \sin(2\pi i(x - t)): 0 \leq i \leq p\right\},
\end{align*}
for $g = 0$, which is a subspace of dimension $k = 2p + 1$, compared again with the dimension $q = k^2 = (2p + 1)^{2}$ of $\ms{H}$.

As a third example, with an elliptic equation, consider the Helmholtz equation in one dimension, that is, the harmonic oscillator $\ms{D}h = h^{\prime\prime} + \omega^{2}h = 0$ on $[0,T]$, which was studied in \cite{koshizuka2025understanding}. Let $\ms{H}$ be the linear span of the constant function and of the functions $\cos(\omega_{j}x)$ and $\sin(\omega_{j}x)$, with $\omega_{j} = j\pi/T$, for $j = 1,\dots,p$, which has dimension $q = 2p + 1$, and assume that $\omega = \omega_{j_{0}}$ for some $1 \leq j_{0} \leq p$. Since $\ms{D}$ multiplies the constant function by $\omega^{2}$ and the functions of frequency $\omega_{j}$ by $\omega^{2} - \omega_{j}^{2}$,
\begin{align*}
	\ms{H}_{\Psi} = \text{span}\left\{\cos(\omega x), \, \sin(\omega x)\right\},
\end{align*}
which has dimension $k = 2$, whatever the number of frequencies in $\ms{H}$. In this case, error bounds that depend on $k$ instead of $q$ were obtained in \cite{koshizuka2025understanding} under Gaussian noise for the points $X_1,\dots,X_n$ satisfying certain assumptions. Proposition \ref{cor_findim} yields the rate $r_{k}$ in $L_{2}(\mu)$ under the weaker assumption $\mathbb{E}[\xi^{2} \mid X] \leq \bar{\sigma}^{2}$ only, and for any distribution $\mu$ of $X$ under which the functions that span $\ms{H}$ are linearly independent in $L_{2}(\mu)$.

More generally, the bounds of this section apply to PISL methods in which $\ms{H}$ is the linear span of a fixed set of functions and the empirical error is penalized by the squared norm of the residual of a linear equation. This is the case of physics-informed estimators based on random features or extreme learning machines \cite{chen2022bridging,dong2021local,dwivedi2020physics}, in which only the weights of the output layer of a neural network are trained, so that, conditionally on the weights of the hidden layers, $\ms{H}$ is a linear space of finite dimension. It is also the case of spatial regression with PDE regularization on a fixed finite element space \cite{arnone2022some,sangalli2021spatial}, of smoothing with penalties based on differential operators on a spline basis \cite{heckman2000penalized,ramsay2007parameter}, of physics-informed polynomial chaos expansions \cite{novak2024physics}, and of physics-informed regression on truncated Fourier or feature expansions \cite{doumeche2024physics2,koshizuka2025understanding}, when no further penalty is considered. In these methods, $k$ is the dimension of the space of solutions of the homogeneous equation that are in $\ms{H}$, and Proposition \ref{cor_findim} quantifies the gain of the penalty as the reduction of the dimension from $q$ to $k$.

Finally, Theorem \ref{main_theorem} also applies to non-convex classes contained in a linear space of finite dimension, since Assumption \ref{small_ball} holds on their convex hull, but the method measures complexity on the hull and cannot exploit the non-convexity. For example, let $\ms{H} = V_{1} \cup \dots \cup V_{m}$ be a union of linear spaces of dimension at most $q_{0}$ with $\ustar \in \ms{H}$. Then $\ms{F} = V_{1} + \dots + V_{m}$ has dimension $q \leq m q_{0}$, $\fstar = \ustar \in \ms{H}$ since $e$ is independent of $X$ with mean zero, so Assumption \ref{concentration} holds trivially, and Lemma \ref{lemma_findim} and Proposition \ref{cor_findim} apply to $\hhat$ with $q = \dim \ms{F}$ and $k = \dim\{v \in \ms{F}: \mathscr{D}v = 0\}$. The physics still reduces the rate from $r_{q}$ to $r_{k}$, but the dimensions are those of $\ms{F}$ rather than of the $V_{i}$ containing $\ustar$, a factor of at most $\sqrt{m}$ in $r_{q}$ when the $V_{i}$ have the same dimension.

For instance, in physics-informed polynomial chaos expansions the degree $p$ of the expansion may be selected among $p \leq p_{\max}$ \cite{novak2024physics}, so $\ms{H}$ is the union of the nested spaces of polynomials of degree at most $p$, whose hull is the largest of them, and the bounds hold with the dimensions of the space of degree $p_{\max}$. In the same methods, a sparse subset of $s$ basis polynomials may be selected by least angle regression \cite{novak2024physics,sharma2024physics}, so $\ms{H}$ is the union of all $s$-dimensional coordinate subspaces of the truncated space, whose hull is the whole truncated space of dimension $q \gg s$. The bounds then hold with $q$ and do not capture the gain of considering only $s$ coordinates, a limitation of the small-ball method when applied to the convex hull.

\section{Extensions and limitations}
\label{SecPoss}

In order to make the presentation clearer, we considered some simplifications that are not strictly necessary, and Theorem \ref{main_theorem} holds without them by straightforward modifications of the proofs. 

First of all, it is not necessary that $Y = \ustar(X) + e$ for $X$ independent of $e$ and the results hold when $(X,Y)$ has more general joint distributions, for example when $X$ and $e$ are not independent. Although the bounds obtained via the small-ball method clearly depend on the joint distribution of $\xi = \fstar(X) - Y$ and $X$, since they depend on $r_{M}$ (see Definition \ref{def_rad}), no strong assumption about this joint distribution is necessary. We refer to \cite{mendelson2015learning} for more details.

Moreover, although we presented the results for differential operators, which is the case of interest in PISL, the proofs only use that $\mathscr{D}$ is linear. Therefore, Theorem \ref{main_theorem} and the results of Section \ref{SecFinDim} hold for any linear operator $\mathscr{D}$ defined on $\ms{F}$ and taking values in $L_{2}(\mu)$, for instance integral operators, so the theory covers prior information expressed by any linear equation $\mathscr{D}\ustar = g$. In the finite-dimensional case, $k$ is then the dimension of the kernel of $\mathscr{D}$ in $\ms{H}$.

As discussed in Section \ref{sec_PISL}, the empirical error could be penalized by other conditions that the data generating function $\ustar$ is known to satisfy. For example, if the boundary condition in \eqref{forward} is known and the operator $\ms{B}$ is linear, then analogous results hold for the penalization function
\begin{align*}
	\Psi(h) = \lVert \ms{D}h - g \rVert_{L_{2}(\mu)}^{2} + \lVert \ms{B}h - b \rVert_{L_{2}(\mu_{\partial})}^{2}.
\end{align*}
Observe that in this case, the same parameter $\lambda$ is considered for both penalty terms when minimising the regularized loss \eqref{PINN_problem}. However, we believe analogous results might be obtained for different parameters at the cost of more technical details. 

Moreover, Tikhonov penalization could also be considered in the loss \eqref{PINN_problem} to increase the solution stability, as was done in \cite{doumeche2024physics,doumeche2024physics2,koshizuka2025understanding,scampicchio2025physics}. Bounds analogous to Theorem \ref{main_theorem} could be obtained by combining the techniques in our proofs with those of \cite{lecue2018regularization} for penalization functions that are a norm. 

In practical applications of PISL, $\Psi(f)$ is actually replaced in $L_{n,\Psi}(f)$ by a numerical approximation $\tPsi(f)$. In general, the results of this paper remain true, with the necessary modifications, when considering these numerical approximations. For instance, if $\tPsi(f)$ is given by a deterministic numerical integration method such that $\sup_{f \in \ms{F}} |\Psi(f) - \tPsi(f)| \leq \beta$, then Theorem \ref{main_theorem} remains true with $\rho + \min\{\Delta_{\Psi},0\}$ replaced by $\rho + \min\{\Delta_{\Psi},0\} - 2\beta$ in \eqref{cond_lambda} and $2\lambda\beta$ added inside the parentheses in \eqref{theorem_improve_eq} Moreover, if $\tPsi$ satisfies the conditions of Remark \ref{rem_gen}, then the results actually hold with $\tPsi$ in place of $\Psi$.

Furthermore, if
\begin{linenomath}
	\begin{align}
		\label{tPsi}
		\tPsi(f) = \frac{1}{m} \sum_{i=1}^{m} (\ms{D}f(Z_{i}) - g(Z_{i}))^{2}
	\end{align}
\end{linenomath}
for a \textit{fixed} sequence $Z_{1},\dots,Z_{m}$ of points in $\Omega$, then the theorem holds by substituting $\Psi$ with $\tPsi$, since $\tPsi$ preserves the necessary properties of $\Psi$ used to prove the theorem (cf. Remark \ref{rem_gen}). In particular, since the proofs of Section \ref{SecFinDim} only use that $\mathscr{D}$ is linear, Lemma \ref{lemma_findim} and Proposition \ref{cor_findim} hold with $\tPsi$ in place of $\Psi$, with $k$ replaced by $\dim\{v \in \ms{H}: \mathscr{D}v(Z_{i}) = 0, \, i = 1,\dots,m\}$ and $c_{\ms{D}}$ by the constant \eqref{cD} computed with an orthonormal basis of $\mathscr{D}(\ms{H})$ in $L_{2}(\mu_{m})$, with $\mu_{m}$ the empirical measure of $Z_{1},\dots,Z_{m}$, and if $Z_{1},\dots,Z_{m}$ is a sample independent of $S_{n}$ they hold conditionally on it, without any approximation bound. An analogous theorem also holds by substituting $\Psi$ with $\tPsi$ in \eqref{tPsi} when $Z_{1},\dots,Z_{m}$ is actually a sample of a random variable $Z$ with distribution $\mu$, independent of $S_{n}$, and the following concentration inequality is satisfied by $\tPsi(f)$ for $f \in \ms{F}$: for all $\delta^{\prime\prime} \in (0,1)$, there exists a sequence $\{\beta_{m}\}$ converging to zero as $m \to \infty$ such that
\begin{linenomath}
	\begin{align}
		\label{conc_Psi}
		Pr\left(\sup\limits_{f \in \ms{F}} |\tPsi(f) - \Psi(f)| \leq \beta_{m}\right) \geq 1 - \delta^{\prime\prime}.
	\end{align}
\end{linenomath}
Inequality \eqref{conc_Psi} is a uniform law of large numbers for the class $\{(\mathscr{D}f - g)^{2}: f \in \ms{F}\}$, which holds, for instance, when this class is bounded in $L_{\infty}(\mu)$ and its Rademacher complexity vanishes as $m \to \infty$ \cite[Theorem~8]{bartlett2002rademacher}, as is the case when it has finite covering numbers \cite[Chapter~13]{concentration}.

In particular, Theorem \ref{main_theorem} holds with the modifications described above, with $\beta_{m}$ in place of $\beta$, and subtracting $\delta^{\prime\prime}$ from the probability of the results. In this stochastic case, which is a special case of Monte Carlo integration, the sample size $m$ is limited only by computational resources, so if concentration inequality \eqref{conc_Psi} holds, then a very good approximation of $\Psi$ may be obtained in a more computationally efficient manner than approximating $\Psi$ by a high fidelity numerical integration method without compromising the convergence of $\hhat$ to $\fstar$. This is the usual method considered in practice \cite{raissi2019physics}.

\subsection{Limitations}

Although the small-ball method is an elegant tool to deduce error rates in convex classes, its application to non-convex classes is quite limited. By a close inspection of the proof of Theorem \ref{theorem_SN} \cite{mendelson2015learning}, we conclude that the minimum assumptions for it to hold are inequality \eqref{projection} about the projection of $Y$ in $\ms{H}$ and that the set $\ms{F} - \fstar \coloneqq \{f - \fstar: f \in \ms{F}\}$ is star-shapedness, that is, for all $0 < \alpha \leq 1$, if $f \in \ms{F} - \fstar$ then $\alpha f \in \ms{F} - \fstar$. Although both conditions hold for convex classes, they are strong conditions for non-convex classes. 

We note that, even though Theorem \ref{main_theorem} applies to non-convex classes, the assumption $\fstar \in \ms{H}$, that the target function of the convex hull is in the space, can be quite strong. For instance, it will hardly hold for neural networks for which it is known that the convex hull is significantly larger than the space of functions generated by them (see \cite[Theorem~2.2]{petersen2021topological}). Moreover, Assumption \ref{small_ball} is not known to hold on the convex hull of neural network classes, whereas Assumption \ref{concentration} holds under the mild conditions of Section \ref{SecConc}. Nevertheless, the theory applies to the random feature method, when the weights of the hidden layers are fixed and only those of the last layer are learnable \cite{chen2022bridging,dong2021local,dwivedi2020physics}.

We also note that the discussion in Section \ref{SecImplications} addressing \textbf{(A)} and \textbf{(B)} is based on comparing upper bounds for the respective error rates. If these upper bounds are not of the correct order, then our comparisons do not reflect the true error rates. However, as discussed in \cite{mendelson2015learning}, the bounds obtained via the small-ball method are often of the right order, so in general our conclusions about the effect of informed penalization should hold.

Finally, the explicit error rates of Section \ref{SecFinDim} are restricted to classes of finite dimension, in which Assumption \ref{small_ball} always holds, while in classes of infinite dimension it does not hold in general. Indeed, Assumption \ref{small_ball} requires that no difference $f - h$ of functions in $\ms{F}$ concentrates on a set of small probability, whereas rich classes, such as balls of Sobolev spaces when $\mu$ has a density, contain functions that differ only on sets of arbitrarily small probability, so that no $\epsilon > 0$ satisfies Assumption \ref{small_ball}. As noted in Section \ref{SecSB}, the small-ball property is only needed for $f - \fstar$ and, by inspecting the proof of Theorem \ref{main_theorem}, only for $f$ in the sub-level set $\ms{F}(\rho,\Psi)$, in which the penalty may prevent functions from concentrating on sets of small probability, so verifying it only on these sets, in the spirit of \cite{Saumard2018}, is a possible route. We leave for future research a better understanding of how the method can be applied in these more general scenarios.

\section{Final remarks}
\label{SecFR}

This paper advances the theoretical understanding of physics-informed statistical learning (PISL) by deducing complexity dependent error rates for estimators regularized by soft penalties induced by linear operators. Based on the small-ball method, we derived high-probability bounds for the $L_{2}$ error of physics-informed estimators, addressing two central questions in PISL: how the performance of soft penalization compares to that of hard constraints, and what is the statistical effect of incorporating physical knowledge into learning procedures.

Our results imply that, under suitable assumptions, the error rates of physics-informed estimators are comparable to those obtained via empirical error minimisation in the respective constrained class, if not equal. This suggests that informed penalization is a statistically sound alternative to hard constraints when they are not computationally feasible. Moreover, again under suitable assumptions, the statistical effect of informed penalization on empirical error minimisation is akin to dimensionality reduction, hence can significantly improve the error rates, especially when the space of solutions to the equation has a low dimension, a finding that is consistent with the literature on specific cases.

We note that, although in linear classes of finite dimension the hard-constrained estimator reduces to least squares on an affine space of dimension $k$, the behaviour of the penalized estimator is not elementary. Proposition \ref{cor_findim} shows that it attains the rate $r_{k}$ for every $\lambda$ above the explicit threshold $\lambda_{\ms{H}}$, which does not depend on $n$, under a bounded conditional second moment of the noise only and for any distribution of $X$, with the same constants in the bounds for the hard constraint, the soft penalty and the unconstrained estimator, so that the comparisons addressing \textbf{(A)} and \textbf{(B)} are exact. The only previous rate depending on $k$ \cite{koshizuka2025understanding} required Gaussian noise.

While we analysed the implications of Theorem \ref{main_theorem} for linear classes of finite dimension, two cases remain open. The first is that of classes of infinite dimension in which the space of solutions of the equation also has infinite dimension and the small-ball assumption does not readily follow. In particular, this requires bounding the product and multiplier processes on the sub-level sets $\ms{F}(\rho,\Psi)$ of the penalty, which could require new techniques. The second is that of non-linear equations, which induce penalties that are not convex (cf. Remark \ref{rem_gen}) and hence are not covered by Theorem \ref{main_theorem}. 

We note that there has been a great effort in the past decades to develop methods to analyse the statistical properties of regularized estimators with norm penalization, while informed penalization has been less studied. Overall, this paper contributes a rigorous foundation for evaluating the statistical properties of physics-informed estimators in classes whose convex hull satisfies a small-ball condition based on the small-ball method, taking a step towards bridging the gap between statistical theory and practical PISL methods in more general frameworks.

\section{Proof of results}
\label{SecProof}

\subsection{Proof of Theorem \ref{main_theorem}}
\label{sec_proof1}

We present some definitions, and state and prove a couple of results that will imply Theorem \ref{main_theorem}. The structure of the argument to prove Theorem \ref{main_theorem} is that of the proof of \cite[Theorem~1.9]{lecue2017regularization}, in which candidates for $\hhat$ are excluded via \eqref{cond_in_F}, first on the level set $\mathscr{F}_{2}$ of the penalty and then outside the sub-level set $\mathscr{F}_{1}$, defined below, by a ray argument from $\fstar$. The bounds (i) and (ii) below on the product and multiplier processes on $\ms{F}_{1}$ are the main result of \cite{mendelson2015learning} applied to the convex class $\ms{F}_{1}$, and the inequality $P_{n}\mathcal{L}_{f} \geq R \, P_{n}\mathcal{L}_{h}$ along rays (cf. \eqref{ray}), in the proof of Lemma \ref{lemma_inF1}, is standard in the small-ball literature \cite{lecue2017regularization,mendelson2015learning}. 

The steps that are new, and due to the form of $\Psi$, are the following. First, the sub-level sets $\ms{F}(\rho,\Psi)$ are defined through the values of $\Psi$, rather than through $\Psi(f - \fstar)$ as in \cite{lecue2017regularization}, and their convexity follows from the linearity of $\mathscr{D}$ (cf. \eqref{F_convex}) rather than from the structural assumptions on $\Psi$ of Section \ref{Sec1.4}. Second, identity \eqref{Rapproxi}, which gives the penalty along a ray from $\fstar$ in closed form and again relies on the linearity of $\mathscr{D}$, replaces the sub-additivity and homogeneity of $\Psi$ used in \cite{lecue2017regularization} to control the penalty outside $\ms{F}_{1}$. Third, the lower bound \eqref{lb_all} on the excess empirical loss over the whole of $\ms{F}_{1}$, together with the pointwise condition \eqref{cond_lambda} on $\lambda$, replaces the uniform condition $\lambda > \lambda_{0}$ of \cite{lecue2017regularization}, in which $\lambda_{0}$ is the supremum over $\rho > 0$ of the ratio between a high-probability bound on the multiplier process on the sub-level set of level $\rho$ and $\rho$ itself, a quantity that may be infinite, and the comparison with a reference target $\hstar$ through Assumption \ref{concentration} allows for non-convex $\ms{H}$. Finally, the explicit bounds of Lemma \ref{lemma_findim2} on the sub-level sets of the penalty, which exploit the decomposition $\ms{H} = W \oplus U$ induced by $\mathscr{D}$, are new, although they are obtained by elementary arguments.

Assume from now on that all the assumptions of Theorem \ref{main_theorem} hold. Let $\mathscr{F}_{1} \coloneqq \mathscr{F}(\rho,\Psi) = \{f \in \mathscr{F}: \Psi(f) - \Psi(\hstar) \leq \rho\}$ and $\mathscr{F}_{2} \coloneqq \{f \in \mathscr{F}: \Psi(f) - \Psi(\hstar) = \rho\}$. Clearly $\hstar \in \ms{F}_{1}$ and, since $\rho > -\Delta_{\Psi} = \Psi(\fstar) - \Psi(\hstar)$, $\fstar \in \ms{F}_{1}$. Furthermore, $\mathscr{F}_{1}$ is a convex set. Indeed, for $0 \leq \alpha \leq 1$ and $f_{1},f_{2} \in \ms{F}$, since $\mathscr{D}$ is linear,
\begin{align}
	\nonumber
	\Psi(\alpha f_{1} + (1-\alpha)f_{2}) &= \lVert \alpha(\ms{D}f_{1} - g) + (1-\alpha)(\ms{D}f_{2} - g) \rVert_{L_{2}(\mu)}^{2}\\
	\label{F_convex}
	&\leq \big(\alpha\lVert \ms{D}f_{1} - g \rVert_{L_{2}(\mu)} + (1-\alpha)\lVert \ms{D}f_{2} - g \rVert_{L_{2}(\mu)}\big)^{2} \leq \alpha\Psi(f_{1}) + (1-\alpha)\Psi(f_{2})
\end{align}
by the triangle inequality and the convexity of $t \mapsto t^{2}$ on $[0,\infty)$, so $\Psi$ is convex. This implies that, for $0 \leq \alpha \leq 1$ and $f_{1},f_{2} \in \ms{F}_{1}$,
\begin{linenomath}
	\begin{align*}
		\Psi(\alpha f_{1} + (1-\alpha)f_{2}) - \Psi(\hstar) &\leq \alpha\Psi(f_{1}) + (1-\alpha)\Psi(f_{2}) - \Psi(\hstar) \\
		&\leq \max\{\Psi(f_{1}),\Psi(f_{2})\} - \Psi(\hstar) \leq \rho
	\end{align*}
\end{linenomath}
so $\alpha f_{1} + (1-\alpha)f_{2} \in \ms{F}_{1}$. Moreover, since the real-valued function $t \mapsto \Psi(\fstar + t(f - \fstar))$ is continuous, the ray $[\fstar,f)$, from $\fstar \in \ms{F}_{1}$ to $f \in \ms{F}\setminus\ms{F}_{1}$ fixed, intersects $\ms{F}_{2}$, since $\ms{F}$ is convex.

Let $\theta = \kappa^{2}\epsilon/16$, set $r_{Q}(\rho) = r_{Q}(\mathscr{F}_{1},\kappa\epsilon/32)$ and $r_{M}(\rho) = r_{M}(\mathscr{F}_{1},\theta/5,\delta/4)$, and recall that $r(\rho) \geq \max\{r_{Q}(\rho),r_{M}(\rho)\}$. Since $\mathscr{F}_{1}$ is a convex class of functions and $\fstar \in \ms{F}_{1}$ is the unique minimiser of the mean squared error in it\footnote{We note that closedness of $\ms{F}_{1}$ is not needed, since $\fstar \in \ms{F}_{1}$ and \eqref{projection} holds on $\ms{F}$}, the main result of \cite{mendelson2015learning} (see also \cite[Theorem~1.6]{lecue2017regularization}) implies, under the small-ball assumption (cf. Assumption \ref{small_ball}), the following: there is an event $\mathcal{A}_{0}$ of probability at least $1 - \delta - 2\exp(-n\epsilon^{2}/2)$ such that for any sample $S_{n} \in \mathcal{A}_{0}$ it holds
\begin{itemize}
	\item[(i)] If $f \in \mathscr{F}_{1}$ and $\lVert f - \fstar \rVert_{L_{2}(\mu)} \geq r_{Q}(\rho)$ then
	\begin{linenomath}
		\begin{align}
			\label{cond_i}
			\frac{1}{n} \sum_{i=1}^{n} (f - \fstar)^{2}(X_{i}) \geq \theta \lVert f - \fstar \rVert_{L_{2}(\mu)}^{2}
		\end{align}
	\end{linenomath}
	\item[(ii)] If $f \in \mathscr{F}_{1}$ then
	\begin{linenomath}
		\begin{align*}
			\left|\frac{1}{n} \sum_{i=1}^{n} \xi_{i}(f - \fstar)(X_{i}) - \mathbb{E}\xi(f - \fstar)(X)\right| \leq \frac{\theta}{4} \max\left\{\lVert f - \fstar \rVert_{L_{2}(\mu)}^{2},r_{M}^{2}(\rho)\right\}.
		\end{align*}
	\end{linenomath} 
\end{itemize} 
In particular, if $f \in \mathscr{F}_{1}$ and $\lVert f - \fstar \rVert_{L_{2}(\mu)} \geq r(\rho) \geq \max\left\{r_{M}(\rho),r_{Q}(\rho)\right\}$ then, by combining \eqref{lb_PnL}, (i) and (ii),
\begin{linenomath}
	\begin{align}
		\label{combine_SB}
		P_{n}\mathcal{L}_{f} \geq \frac{\theta}{2} \lVert f - \fstar \rVert_{L_{2}(\mu)}^{2}.
	\end{align}
\end{linenomath}

Moreover, if $f \in \mathscr{F}_{1}$ and $\lVert f - \fstar \rVert_{L_{2}(\mu)} \leq r(\rho)$, then, by \eqref{lb_PnL} and (ii),
\begin{align}
	\label{lb}
	P_{n}\mathcal{L}_{f} \geq -2|P_{n}\mathcal{M}_{f - \fstar}| \geq -\frac{\theta}{2}\max\{\lVert f - \fstar \rVert_{L_{2}(\mu)}^{2},r_{M}^{2}(\rho)\} \geq -\frac{\theta}{2}r^{2}(\rho).
\end{align}
Combining \eqref{combine_SB} with \eqref{lb} implies that, for all $S_{n} \in \mathcal{A}_{0}$ and $f \in \mathscr{F}_{1}$,
\begin{linenomath}
	\begin{align}
		\label{lb_all}
		P_{n}\mathcal{L}_{f} \geq - \frac{\theta}{2} \, r^{2}(\rho).
	\end{align}
\end{linenomath}

Finally, by Assumption \ref{concentration}, there exists an event $\mc{A}_{1}$ of probability at least $1 - \delta^{\prime}$ such that
\begin{linenomath}
	\begin{align}
		\label{cond_d_ineq}
		&L_{n}(\hstar) - L_{n}(\fstar) \leq \Delta_{L} + d_{n}
	\end{align}
\end{linenomath}
for all $S_{n} \in \mc{A}_{1}$. Set $\mathcal{A} = \mathcal{A}_{0} \cap \mc{A}_{1}$ and observe that $\mathcal{A}$ has probability at least $1 - \delta - \delta^{\prime} - 2\exp(-n\epsilon^{2}/2)$.

In view of \eqref{cond_min}, we first show that for $S_{n} \in \mathcal{A}$  and all $f \in \mathscr{F}_{2}$,
\begin{linenomath}
	\begin{align*}
		P_{n}\mathcal{L}_{f} > \lambda (\Psi(\hstar) - \Psi(f)) + L_{n}(\hstar) - L_{n}(\fstar).
	\end{align*}
\end{linenomath}

\begin{lemma}
	\label{lemma_notinF2}
	Under the assumptions of Theorem \ref{main_theorem}, for all $S_{n} \in \mathcal{A}$ and $f \in \mathscr{F}_{2}$,
	\begin{linenomath}
		\begin{align}
			\label{exclude}
			P_{n}\mathcal{L}_{f} > \lambda (\Psi(\hstar) - \Psi(f)) + L_{n}(\hstar) - L_{n}(\fstar).
		\end{align}
	\end{linenomath}
\end{lemma}
\begin{proof}
	Fix $S_{n} \in \mathcal{A}$ and $f \in \mathscr{F}_{2}$. Since $\Psi(f) - \Psi(\hstar) = \rho$, it follows from \eqref{lb_all}, \eqref{cond_d_ineq} and \eqref{cond_lambda} that
	\begin{linenomath}
		\begin{align*}
			&P_{n}\mathcal{L}_{f} + \lambda(\Psi(f) - \Psi(\hstar)) - [L_{n}(\hstar) - L_{n}(\fstar)]\\
			&\geq -\frac{\theta}{2}r^{2}(\rho) + \lambda\rho - \Delta_{L} - d_{n} \geq -\frac{\theta}{2}r^{2}(\rho) + \lambda\left(\rho + \min\{\Delta_{\Psi},0\}\right) - \Delta_{L} - d_{n} > 0,
		\end{align*}
	\end{linenomath}
	recalling that $\theta/2 = \kappa^{2}\epsilon/32$.
\end{proof}

We now show that for $S_{n} \in \mathcal{A}$ \eqref{exclude} also holds for all $f \in \mathscr{F}\setminus\mathscr{F}_{1}$. From this it follows by \eqref{cond_in_F} that $\hhat \notin (\mathscr{F}\setminus\mathscr{F}_{1} \cap \ms{H})$ for $S_{n} \in \mathcal{A}$.

\begin{lemma}
	\label{lemma_inF1}
	Under the assumptions of Theorem \ref{main_theorem}, for all $S_{n} \in \mathcal{A}$ and $f \in \mathscr{F}\setminus\mathscr{F}_{1}$,
	\begin{linenomath}
		\begin{align*}
			P_{n}\mathcal{L}_{f} > \lambda (\Psi(\hstar) - \Psi(f)) + L_{n}(\hstar) - L_{n}(\fstar).
		\end{align*}
	\end{linenomath}
\end{lemma}
\begin{proof}
	Fix $S_{n} \in \mathcal{A}$ and $f \in \mathscr{F}\setminus\mathscr{F}_{1}$. Since $\Psi(\fstar) - \Psi(\hstar) < \rho < \Psi(f) - \Psi(\hstar)$ and $t \mapsto \Psi(\fstar + t(f - \fstar))$ is continuous, there exist $h \in \mathscr{F}_{2}$ and $R > 1$ such that $f = \fstar + R(h - \fstar)$. Since $\mathscr{D}$ is linear, 
	\begin{align*}
		\Psi(f) = \lVert \mathscr{D}\fstar - g \rVert_{L_{2}(\mu)}^{2} + 2R\langle \mathscr{D}\fstar - g, \mathscr{D}h - \mathscr{D}\fstar \rangle_{L_{2}(\mu)} + R^{2}\lVert \mathscr{D}h - \mathscr{D}\fstar \rVert_{L_{2}(\mu)}^{2}.
	\end{align*}	
	Evaluating this expansion at $R = 1$ yields
	\begin{align*}
		2\langle \mathscr{D}\fstar - g, \mathscr{D}h - \mathscr{D}\fstar \rangle_{L_{2}(\mu)} = \Psi(h) - \Psi(\fstar) - \lVert \mathscr{D}h - \mathscr{D}\fstar \rVert_{L_{2}(\mu)}^{2}
	\end{align*}
	and substituting back we obtain the identity
	\begin{align}
		\label{Rapproxi}
		\Psi(f) = \Psi(\fstar) + R\left[\Psi(h) - \Psi(\fstar)\right] + R(R-1)\lVert \mathscr{D}h - \mathscr{D}\fstar \rVert_{L_{2}(\mu)}^{2}.
	\end{align}
	
	Since $R(R - 1) \geq 0$ and, recalling the definition of $\Delta_{\Psi}$ (cf. \eqref{gaps}), $\Psi(h) - \Psi(\fstar) = \rho + \Delta_{\Psi}$ as $h \in \mathscr{F}_{2}$, identity \eqref{Rapproxi} yields $\Psi(f) - \Psi(\hstar) = [\Psi(f) - \Psi(\fstar)] - \Delta_{\Psi} \geq R(\rho + \Delta_{\Psi}) - \Delta_{\Psi}$, and hence
	\begin{linenomath}
		\begin{align*}
			\lambda\left(\Psi(f) - \Psi(\hstar)\right) \geq R\lambda\left(\rho + (1 - R^{-1})\Delta_{\Psi}\right) \geq R\lambda\left(\rho + \min\{\Delta_{\Psi},0\}\right),
		\end{align*}
	\end{linenomath}
	since $(1 - R^{-1})\Delta_{\Psi} \geq \min\{\Delta_{\Psi},0\}$ for $R > 1$. Furthermore, again since $R > 1$,
	\begin{linenomath}
		\begin{align}
			\label{ray}
			P_{n}\mathcal{L}_{f} = \frac{R^{2}}{n} \sum_{i=1}^{n} (h - \fstar)^{2}(X_{i}) + \frac{2R}{n} \sum_{i=1}^{n} \xi_{i}(h  - \fstar)(X_{i}) \geq R P_{n}\mathcal{L}_{h},
		\end{align}
	\end{linenomath}
	and, by \eqref{cond_d_ineq}, $-[L_{n}(\hstar) - L_{n}(\fstar)] \geq -(\Delta_{L} + d_{n}) \geq -R(\Delta_{L} + d_{n})$, as $\Delta_{L} + d_{n} \geq 0$. Combining the last three bounds with \eqref{lb_all}, applied to $h \in \mathscr{F}_{2} \subset \mathscr{F}_{1}$, we conclude that
	\begin{linenomath}
		\begin{align*}
			&P_{n}\mathcal{L}_{f} + \lambda\left(\Psi(f) - \Psi(\hstar)\right) - [L_{n}(\hstar) - L_{n}(\fstar)]\\
			&\geq R\left[P_{n}\mathcal{L}_{h} + \lambda\left(\rho + \min\{\Delta_{\Psi},0\}\right) - \Delta_{L} - d_{n}\right]\\
			&\geq R\left[-\frac{\theta}{2}r^{2}(\rho) + \lambda\left(\rho + \min\{\Delta_{\Psi},0\}\right) - \Delta_{L} - d_{n}\right] > 0
		\end{align*}
	\end{linenomath}
	in which the last inequality is due to \eqref{cond_lambda} recalling that $\theta = \kappa^2\epsilon/16$.
\end{proof}

We are now in a position to prove Theorem \ref{main_theorem}.

\begin{proof}[Proof of Theorem \ref{main_theorem}]
	It follows from Lemmas \ref{lemma_notinF2} and \ref{lemma_inF1}, and from \eqref{cond_in_F}, that $\hhat \in \mathscr{F}_{1} \cap \ms{H}$ for all samples $S_{n} \in \mathcal{A}$, therefore with probability at least $1 - \delta - \delta^{\prime} - 2\exp(-n\epsilon^{2}/2)$. Moreover, fix $S_{n} \in \mathcal{A}$ and assume that $\lVert \hhat - \fstar \rVert_{L_{2}(\mu)} \geq r(\rho)$. Since $\hhat \in \mathscr{F}_{1}$, it follows from \eqref{cond_min}, \eqref{main_inequality}, \eqref{combine_SB} and \eqref{cond_d_ineq} that
	\begin{linenomath}
		\begin{align*}
			\frac{\theta}{2}\lVert \hhat - \fstar \rVert_{L_{2}(\mu)}^{2} &\leq P_{n}\mathcal{L}_{\hhat} = L_{n}(\hhat) - L_{n}(\hstar) + L_{n}(\hstar) - L_{n}(\fstar)\\
			&\leq \lambda(\Psi(\hstar) - \Psi(\hhat)) + \Delta_{L} + d_{n} \leq \lambda(\Psi(\hstar) - \Psi_{\ms{H}}) + \Delta_{L} + d_{n},
		\end{align*}
	\end{linenomath}
	since $\Psi(\hhat) \geq \Psi_{\ms{H}}$. As $2/\theta = 32/(\kappa^{2}\epsilon)$, \eqref{theorem_improve_eq} follows. Finally, if $\fstar \in \ms{H}$, then $\fstar \in \Hstar(0)$ and, taking $\hstar = \fstar$, $L_{n}(\hstar) - L_{n}(\fstar) = 0 = L(\hstar) - L(\fstar)$, so Assumption \ref{concentration} holds with $d_{n} = 0$ and $\delta^{\prime} = 0$, and $\Delta_{L} = \Delta_{\Psi} = 0$.
\end{proof}

\subsection{Proof of Proposition \ref{cor_findim}}

We first prove a more general lemma, that also applies to the misspecified case, from which Lemma \ref{lemma_findim} follows as a special case. Recall that we denote $\Psi_{\ms{H}} \coloneqq \inf_{h \in \ms{H}}\Psi(h)$, which is attained at some $h_{g} \in \ms{H}$ since it is a finite-dimensional least squares problem, and let $\Delta_{\ms{H}} \coloneqq \Psi(\fstar) - \Psi_{\ms{H}}$. For $\rho \geq 0$ let	
\begin{linenomath}
	\begin{equation}
		\label{t_rho}
		t_{\rho} \coloneqq \frac{\sqrt{\Delta_{\ms{H}} + \rho} + \sqrt{\Delta_{\ms{H}}}}{c_{\ms{D}}}.
	\end{equation}
\end{linenomath}

\begin{lemma}
	\label{lemma_findim2}
	Consider all the assumptions of Lemma \ref{lemma_findim}, but let $\Psi(\fstar) \geq 0$ be arbitrary. Then, for every $\rho \geq 0$, $r_{Q}(\rho) = 0$ and
	\begin{linenomath}
		\begin{equation}
			\label{r_V2}
			r_{M}(\rho) \leq r_{\ms{H}}(\rho) \coloneqq \max\left\{r_{k}, \ \left(r_{q-k} \, t_{\rho}\right)^{1/2}\right\}.
		\end{equation}
	\end{linenomath}
	In particular, $r_{\ms{H}}$ is an admissible choice of $r(\cdot)$ in Theorem \ref{main_theorem}, Lemma \ref{lemma_findim} is the case $\Delta_{\ms{H}} = 0$ and $$r_{\ms{H}}(0) = \max\{r_{k}, (2r_{q-k}\sqrt{\Delta_{\ms{H}}}/c_{\ms{D}})^{1/2}\}.$$
\end{lemma}
\begin{proof}[Proof of Lemma \ref{lemma_findim2}]
	Fix an orthonormal basis $\psi_{1},\dots,\psi_{q}$ of $\ms{H}$ in $L_{2}(\mu)$ such that $\psi_{1},\dots,\psi_{k}$ span $W$ and $\psi_{k+1},\dots,\psi_{q}$ span $U$, and write $\psi \coloneqq (\psi_{1},\dots,\psi_{q})$, $\psi^{W} \coloneqq (\psi_{1},\dots,\psi_{k})$ and $\psi^{U} \coloneqq (\psi_{k+1},\dots,\psi_{q})$. Let $h_{g}$ and $\Psi_{\ms{H}}$ be as above and, for $h \in \ms{H}$, write $h - h_{g} = w_{h} + u_{h}$ with $w_{h} \in W$ and $u_{h} \in U$.
	
	Since $h_{g}$ minimises $\Psi(h) = \lVert \mathscr{D}h - g \rVert_{L_{2}(\mu)}^{2}$ over the linear space $\ms{H}$, $\mathscr{D}h_{g}$ is the orthogonal projection of $g$ onto the subspace $\mathscr{D}(\ms{H})$ of $L_{2}(\mu)$, so $\mathscr{D}h_{g} - g$ is orthogonal to $\mathscr{D}(\ms{H})$. Moreover, $\mathscr{D}(h - h_{g}) = \mathscr{D}w_{h} + \mathscr{D}u_{h} = \mathscr{D}u_{h}$, as $\mathscr{D}$ is linear and vanishes on $W$. Hence, for every $h \in \ms{H}$,
	\begin{linenomath}
		\begin{equation}
			\label{tube}
			\Psi(h) = \lVert \mathscr{D}(h - h_{g}) + (\mathscr{D}h_{g} - g) \rVert_{L_{2}(\mu)}^{2} = \lVert \mathscr{D}u_{h} \rVert_{L_{2}(\mu)}^{2} + \Psi_{\ms{H}}.
		\end{equation}
	\end{linenomath}
	Fix $\rho \geq 0$ and let $f \in \ms{F}(\rho,\Psi)$, so that $\Psi(f) \leq \Psi(\fstar) + \rho = \Psi_{\ms{H}} + \Delta_{\ms{H}} + \rho$, from which follows that, by \eqref{tube}, $\lVert \mathscr{D}u_{f} \rVert_{L_{2}(\mu)} \leq \sqrt{\Delta_{\ms{H}} + \rho}$ and $\lVert \mathscr{D}u_{\fstar} \rVert_{L_{2}(\mu)} = \sqrt{\Delta_{\ms{H}}}$. Therefore, as $f - \fstar = (f - h_{g}) - (\fstar - h_{g}) = (w_{f} - w_{\fstar}) + (u_{f} - u_{\fstar})$, the component of $f - \fstar$ in $U$ is $u_{f} - u_{\fstar}$, which satisfies
	\begin{align}
		\label{boundDU}
		\lVert \mathscr{D}(u_{f} - u_{\fstar}) \rVert_{L_{2}(\mu)} \leq c_{\ms{D}} \, t_{\rho}
	\end{align}
	by the triangle inequality. Consequently, every $f \in \ms{F}(\rho,\Psi) \cap rD_{\fstar}$ can be written as $f - \fstar = \langle \beta, \psi \rangle$ with $\beta = (\beta^{W},\beta^{U}) \in \mathbb{R}^{k} \times \mathbb{R}^{q - k}$, $\lVert \beta \rVert = \lVert f - \fstar \rVert_{L_{2}(\mu)} \leq r$ and $\langle \beta^{U}, \psi^{U} \rangle = u_{f} - u_{\fstar}$, in which $\lVert \cdot \rVert$ is the Euclidean norm.
	
	We first bound the product process. Let $Z \coloneqq \frac{1}{n}\sum_{i=1}^{n}\varepsilon_{i}\psi(X_{i}) \in \mathbb{R}^{q}$. For $f$ as above, by the Cauchy-Schwarz inequality, 
	\begin{equation*}
		\left|\frac{1}{n}\sum_{i=1}^{n}\varepsilon_{i}(f - \fstar)(X_{i})\right| = |\langle \beta, Z \rangle| \leq r\lVert Z \rVert.
	\end{equation*}
	Since the $\varepsilon_{i}$ are independent, mean zero and independent of $S_{n}$, the cross terms in $\mathbb{E}\lVert Z \rVert^{2}$ vanish and, by orthonormality, $\mathbb{E}\lVert \psi(X) \rVert^{2} = \sum_{j=1}^{q}\mathbb{E}\psi_{j}(X)^{2} = q$, so that $\mathbb{E}\lVert Z \rVert \leq (\mathbb{E}\lVert Z \rVert^{2})^{1/2} = \sqrt{q/n} \leq \kappa\epsilon/32$, as $n \geq 1024 \, q/(\kappa\epsilon)^{2}$. Therefore
	\begin{equation*}
		\mathbb{E}\sup_{f \in \ms{F}(\rho,\Psi) \cap rD_{\fstar}}\left|\frac{1}{n}\sum_{i=1}^{n}\varepsilon_{i}(f - \fstar)(X_{i})\right| \leq (\kappa\epsilon/32) \, r
	\end{equation*}
	for every $r > 0$, so the infimum in the definition of $r_{Q}(\rho) = r_{Q}(\ms{F}(\rho,\Psi),\kappa\epsilon/32)$ (cf. Definition \ref{def_rad}) is taken over all $r > 0$, and $r_{Q}(\rho) = 0$.
	
	We turn to the multiplier process. Let $A \coloneqq \frac{1}{\sqrt{n}}\sum_{i=1}^{n}\varepsilon_{i}\xi_{i}\psi(X_{i}) \in \mathbb{R}^{q}$, with coordinates $A^{W} \in \mathbb{R}^{k}$ and $A^{U} \in \mathbb{R}^{q - k}$ corresponding to $\psi^{W}$ and $\psi^{U}$. As above, the cross terms in $\mathbb{E}\lVert A^{W} \rVert^{2}$ vanish and, since $\mathbb{E}[\xi^{2} \mid X] \leq \bar{\sigma}^{2}$ and $\mathbb{E}\lVert \psi^{W}(X) \rVert^{2} = k$, 
	\begin{equation}
		\label{eq1}
		\mathbb{E}\lVert A^{W} \rVert^{2} = \mathbb{E}\left[\xi^{2}\lVert \psi^{W}(X) \rVert^{2}\right] = \mathbb{E}\left[\mathbb{E}[\xi^{2} \mid X]\lVert \psi^{W}(X) \rVert^{2}\right] \leq \bar{\sigma}^{2}k.
	\end{equation}
	Analogously, for $G \in \mathbb{R}^{q-k}$ with $G_{j} \coloneqq \frac{1}{\sqrt{n}}\sum_{i=1}^{n}\varepsilon_{i}\xi_{i}(T_{\ms{D}}e_{j})(X_{i})$,
	\begin{equation}
		\label{eq2}
		\mathbb{E}\lVert G \rVert^{2} \leq \bar{\sigma}^{2}\sum_{j}\lVert T_{\ms{D}}e_{j} \rVert_{L_{2}(\mu)}^{2} = \bar{\sigma}^{2}(q - k)/c_{\ms{D}}^{2}
	\end{equation}
	recalling the definitions of $T_{\ms{D}}$ and $c_{\ms{D}}$ in \eqref{cD}. Since, for $f$ as above, $u_{f} - u_{\fstar} = \sum_{j} \alpha_{j} \, T_{\ms{D}}e_{j}$ with $\alpha_{j} \coloneqq \langle \mathscr{D}(u_{f} - u_{\fstar}), e_{j} \rangle_{L_{2}(\mu)}$, we have $\langle \beta^{U}, A^{U} \rangle = \langle \alpha, G \rangle$, with $\lVert \alpha \rVert \leq c_{\ms{D}} \, t_{\rho}$ by Cauchy-Schwartz inequality and \eqref{boundDU}.
	
	By Markov's inequality, \eqref{eq1} and \eqref{eq2}, each of the events
	\begin{align*}
		\{\lVert A^{W} \rVert > \bar{\sigma}\sqrt{8k/\delta}\} & & \text{ and } & & \{\lVert G \rVert > (\bar{\sigma}/c_{\ms{D}})\sqrt{8(q - k)/\delta}\}
	\end{align*}
	has probability at most $\delta/8$. By the union bound, there exists an event $\mathcal{A}$ with probability at least $1 - \delta/4$ such that, for $S_{n} \in \mathcal{A}$, $\lVert A^{W} \rVert \leq \bar{\sigma}\sqrt{8k/\delta}$ and $\lVert G \rVert \leq (\bar{\sigma}/c_{\ms{D}})\sqrt{8(q - k)/\delta}$. Therefore, in this event, for all $s > 0$, by the Cauchy-Schwarz inequality applied to each block of coordinates, recalling the definition of $\phi_{n}$ in Definition \ref{def_rad},
	\begin{linenomath}
		\begin{align*}
			\phi_{n}(\ms{F}(\rho,\Psi),\fstar,s) &= \sup\limits_{\lVert \beta \rVert \leq s} \left|\langle \beta, A \rangle\right| \leq \sup\limits_{\lVert \beta \rVert \leq s} \left(\lVert \beta^{W} \rVert\lVert A^{W} \rVert + \lVert \alpha \rVert\lVert G \rVert\right) \\
			&\leq s\bar{\sigma}\sqrt{\frac{8k}{\delta}} + t_{\rho}\bar{\sigma}\sqrt{\frac{8(q - k)}{\delta}}.
		\end{align*}
	\end{linenomath}
	The right-hand side of the expression above is at most $(\kappa^{2}\epsilon/80) \, s^{2}\sqrt{n}$ whenever each of its two terms is at most $(\kappa^{2}\epsilon/160) \, s^{2}\sqrt{n}$, that is, whenever 
	\begin{align*}
		s \geq \frac{160\sqrt{8}}{\kappa^{2}\epsilon} \, \bar{\sigma}\sqrt{\frac{k}{\delta n}} = r_{k} & & \text{and} & & s^{2} \geq \frac{160\sqrt{8}}{\kappa^{2}\epsilon} \, \bar{\sigma}\sqrt{\frac{q - k}{\delta n}} \, t_{\rho} = r_{q-k}t_{\rho}.
	\end{align*}
	Hence, by the definition of $r_{M}(\rho) = r_{M}(\ms{F}(\rho,\Psi),\kappa^{2}\epsilon/80,\delta/4)$ (cf. Definition \ref{def_rad}), $r_{M}(\rho) \leq \max\{r_{k},(r_{q-k}t_{\rho})^{1/2}\} = r_{\ms{H}}(\rho)$.
	
	Since $r_{\ms{H}}(\rho) \geq \max\{r_{M}(\rho),r_{Q}(\rho)\}$ for every $\rho \geq 0$, $r_{\ms{H}}$ is an admissible choice of $r(\cdot)$ in Theorem \ref{main_theorem}, and the value of $r_{\ms{H}}(0)$ follows from $t_{0} = 2\sqrt{\Delta_{\ms{H}}}/c_{\ms{D}}$. Finally, if $\Delta_{\ms{H}} = 0$, in particular under the assumption $\Psi(\fstar) = 0$, then $t_{\rho} = \sqrt{\rho}/c_{\ms{D}}$ and \eqref{r_V2} reduces to \eqref{r_V}, which proves Lemma \ref{lemma_findim}.
\end{proof}

\begin{remark}
	\label{remarkH}
	Repeating the argument of the proof of Lemma \ref{lemma_findim2} with $\psi^{W} = \psi$, that is, without using the penalty, shows that, if $n \geq 1024 \, q/(\kappa\epsilon)^{2}$, then $r_{M}(\ms{H},\kappa^{2}\epsilon/80,\delta/4) \leq r_{q}$ and $r_{Q}(\ms{H},\kappa\epsilon/32) = 0$, so Theorem \ref{theorem_SN} applies to learning on $\ms{H}$ with $r_{q}$ on the right-hand side of \eqref{bound_sb} in this case.
\end{remark}

\begin{remark}
	\label{sub_Gaussian}
	The dependence of $r_{j}$ on $\delta$ through $1/\sqrt{\delta}$, rather than through $\sqrt{\log(1/\delta)}$ as under sub-Gaussian noise, is the price for assuming only a bounded conditional second moment of the noise, which allows heavy-tailed distributions. Under stronger moment or tail assumptions on $\xi$, Markov's inequality in the proof of Lemma \ref{lemma_findim2} could be replaced by sharper concentration inequalities, yielding a logarithmic dependence on $\delta$.
\end{remark}

\begin{proof}[Proof of Proposition \ref{cor_findim}]
	Fix $\lambda > 0$. We first prove that $r_{\ms{H}}(\rho_{\lambda}) \leq s_{\lambda}$. By \eqref{r_V}, since $r_{k} \leq s_{\lambda}$ holds by the definition of $s_{\lambda}$, it suffices to show that $r_{q-k}\sqrt{\rho_{\lambda}}/c_{\ms{D}} \leq s_{\lambda}^{2}$. Now, since $\sqrt{\rho_{\lambda}}/c_{\ms{D}} = \kappa\sqrt{\epsilon} \, s_{\lambda}/(4c_{\ms{D}}\sqrt{\lambda})$ and $C_{2} = C_{1}\kappa\sqrt{\epsilon}/4$,
	\begin{linenomath}
		\begin{equation*}
			r_{q-k}\frac{\sqrt{\rho_{\lambda}}}{c_{\ms{D}}} = \frac{C_{1}\kappa\sqrt{\epsilon}}{4} \, \bar{\sigma}\sqrt{\frac{q-k}{\delta n}} \, \frac{s_{\lambda}}{c_{\ms{D}}\sqrt{\lambda}} = \frac{C_{2}}{c_{\ms{D}}}\bar{\sigma}\sqrt{\frac{q - k}{\delta n \lambda}} \, s_{\lambda} \leq s_{\lambda}^{2}
		\end{equation*}
	\end{linenomath}
	by the definition of $s_{\lambda}$. Hence $\lambda\rho_{\lambda} = \kappa^{2}\epsilon \, s_{\lambda}^{2}/16 > \kappa^{2}\epsilon \, s_{\lambda}^{2}/32 > \kappa^{2}\epsilon \, r_{\ms{H}}^{2}(\rho_{\lambda})/32$, which is condition \eqref{cond_lambda} with $r = r_{\ms{H}}$, and Theorem \ref{main_theorem}, with $\Psi(\fstar) = \Psi_{\ms{H}} = 0$, yields \eqref{cor_bound} since $r_{\ms{H}}(\rho_{\lambda}) \leq s_{\lambda}$. Finally, $s_{\lambda} = r_{k}$ if, and only if, $(C_{2}/c_{\ms{D}})\sqrt{(q - k)/\lambda} \leq C_{1}\sqrt{k}$, that is, $\lambda \geq (C_{2}/C_{1})^{2}(q - k)/(c_{\ms{D}}^{2}k) = \lambda_{\ms{H}}$, since $(C_{2}/C_{1})^{2} = \kappa^{2}\epsilon/16$.
\end{proof}

\section*{Acknowledgements}

This work was partially funded by grants \#22/06211-2 and \#23/00256-7 S\~ao Paulo Research Foundation (FAPESP). Part of this work was carried when the author was a visiting scholar at the Department of Electrical and Computer Engineering, Texas A\&M University, and a postdoctoral researcher at the Institute of Mathematics and Statistics, University of S\~ao Paulo. The author used Claude (Anthropic) to assist in revising the text and in checking the proofs of this paper. All results were verified by the author. The author assumes responsibility for all content.

\bibliographystyle{plain}
\bibliography{Ref}

\end{document}